\documentclass[letterpaper]{article}
\usepackage{uai2019}
\usepackage[margin=1in]{geometry}
\usepackage{amsmath}
\usepackage{url}  
\usepackage{graphicx}  
\usepackage{subcaption}
\usepackage{amsfonts}
\usepackage{makecell}
\usepackage{multirow}
\usepackage{bbm}
\usepackage[T1]{fontenc}
\usepackage{algorithm2e}
\usepackage{color}

\usepackage{amsthm}
\newtheorem{definition}{Definition}

\newtheorem{lemma}{Lemma}
\newtheorem{theorem}{Theorem}
\newtheorem{proposition}{Proposition}
\newtheorem{corollary}{Corollary}
\usepackage{xcolor}

\usepackage[round]{natbib} 
\setcitestyle{authoryear, open={(},close={)},citesep={;},aysep=}
\usepackage{times}

\title{CCMI : Classifier based Conditional Mutual Information Estimation }

\author{\vspace{-0.2in}} 

\author{ {\bf Sudipto Mukherjee, Himanshu Asnani, Sreeram Kannan} \\
Department of Electrical and Computer Engineering, \\
University of Washington, Seattle, WA.\\
\{sudipm, asnani, ksreeram\}@uw.edu \\
}

\begin{document}
\maketitle

\begin{abstract}
Conditional Mutual Information (CMI) is a measure of conditional dependence between random variables X and Y, given another random variable Z. It can be used to quantify conditional dependence among variables in many data-driven inference problems such as graphical models, causal learning, feature selection and time-series analysis. While k-nearest neighbor ($k$NN) based estimators as well as kernel-based methods have been widely used for CMI estimation, they suffer severely from the curse of dimensionality. In this paper, we leverage advances in classifiers and generative models to design methods for CMI estimation. Specifically, we introduce an estimator for KL-Divergence based on the likelihood ratio by training a classifier to distinguish the observed joint distribution from the product distribution. We then show how to construct several CMI estimators using this basic divergence estimator by drawing ideas from conditional generative models. We demonstrate that the estimates from our proposed approaches do not degrade in performance with increasing dimension and obtain significant improvement over the widely used KSG estimator. Finally, as an application of accurate CMI estimation, we use our best estimator for conditional independence testing and achieve superior performance than the state-of-the-art tester on both simulated and real data-sets.
\end{abstract}

\vspace{-0.3in}
\section{Introduction}
\vspace{-0.1in}

Conditional mutual information (CMI) is a fundamental information theoretic quantity that extends the nice properties of mutual information (MI) in conditional settings. For three continuous random variables, $X$, $Y$ and $Z$, the conditional mutual information is defined as:
\begin{equation*}
I(X;Y|Z) = \iiint p(x,y,z) \log\frac{p(x,y,z)}{p(x,z)p(y|z)} dx dy dz
\label{cmi-int}
\end{equation*}

assuming that the distributions admit the respective densities $p(\cdot)$.
 One of the striking features of MI and CMI is that they can capture non-linear dependencies between the variables. In scenarios where Pearson correlation is zero even when the two random variables are dependent, mutual information can recover the truth. 
 Likewise, in the sense of conditional independence for the case of three random variables $X$,$Y$ and $Z$, conditional mutual information provides strong guarantees, i.e., $X \perp Y | Z \iff I(X;Y|Z) = 0$. 

The conditional setting is even more interesting as dependence between $X$ and $Y$ can potentially change based on how they are connected to the conditioning variable. For instance, consider a simple Markov chain where $X \rightarrow Z \rightarrow Y$. Here, $X \perp Y |Z$. But a slightly different relation $X \rightarrow Z \leftarrow Y$ has $X \not\perp Y |Z$, even though $X$ and $Y$ may be independent as a pair. It is a well known fact in Bayesian networks that a node is independent of its non-descendants given its parents. CMI goes beyond stating whether the pair $(X,Y)$ is conditionally dependent or not. It also provides a quantitative strength of dependence. 
\vspace{-0.1in}

\subsection{Prior Art}
\vspace{-0.1in}

The literature is replete with works aimed at applying CMI for data-driven knowledge discovery. \cite{fleuret} used CMI for fast binary feature selection to improve classification accuracy. \cite{loeckx} improved non-rigid image registration by using CMI as a similarity measure instead of global mutual information. CMI has been used to infer gene-regulatory networks \citep{liang2008gene} or protein modulation \citep{giorgi2014protein} from gene expression data. Causal discovery \citep{Licausal, hlinka, vejmelka2008inferring} is yet another application area of CMI estimation.

Despite its wide-spread use, estimation of conditional mutual information remains a challenge. One naive method may be to estimate the joint and conditional densities from data and plug it into the expression for CMI. But density estimation is not sample efficient and is often more difficult than estimating the quantities directly. The most widely used technique expresses CMI in terms of appropriate arithmetic of differential entropy estimators (referred to here as $\Sigma H$ estimator): 
$I(X;Y|Z) = h(X,Z) + h(Y,Z) - h(Z) - h(X,Y,Z)
\label{entropy-sum}$, where $h(X) = - \int\limits_{\mathcal{X}} p(x) \log p(x) \, dx$ is known as the differential entropy.

The differential entropy estimation problem has been studied extensively by \cite{beirlant1997nonparametric, nemenman2002entropy, miller2003new, lee2010sample, lesniewicz2014expected, sricharan2012estimation, singh2014exponential} and can be estimated either based on kernel-density \citep{kandasamy2015nonparametric,gao2016breaking} or $k$-nearest-neighbor estimates \citep{sricharan2013ensemble,jiao2017nearest, pal2010estimation, kozachenko1987sample, singh2003nearest, singh2016finite}. Building on top of $k$-nearest-neighbor estimates and breaking the paradigm of $\Sigma H$ estimation, a coupled estimator (which we address henceforth as KSG) was proposed by \cite{kraskov}. It generalizes to mutual information,  conditional mutual information as well as for other multivariate information measures, including estimation in scenarios when the distribution can be mixed \citep{runge2017conditional, frenzel2007partial, gao2017mixture, gao2018demystifying, vejmelka2008inferring, gdmnips2018}.

The $k$NN approach has the advantage that it can naturally adapt to the data density and does not require extensive tuning of kernel band-widths. However, all these approaches suffer from the curse of dimensionality and are unable to scale well with dimensions. Moreover, \cite{gao2015efficient} showed that exponentially many samples are required (as MI grows) for the accurate estimation using $k$NN based estimators. This brings us to the central motivation of this work : \textbf{\textit{Can we propose estimators for conditional mutual information that estimate well even in high dimensions ?}} 
\vspace{-0.2in}
\subsection{Our Contribution}
\vspace{-0.1in}

In this paper, we explore various ways of estimating CMI by leveraging tools from classifiers and generative models. To the best of our knowledge, this is the first work that deviates from the framework of $k$NN and kernel based CMI estimation and introduces neural networks to solve this problem. 

The main contributions of the paper can be summarized as follows :

\textbf{Classifier Based MI Estimation:} We propose a novel KL-divergence estimator based on classifier two-sample approach that is more stable and performs superior to the recent neural methods \citep{belghazi}. \newline
\textbf{Divergence Based CMI Estimation:} We express CMI as the KL-divergence between two distributions $p_{xyz} = p(z)p(x|z)p(y|x,z)$ and $q_{xyz} = p(z)p(x|z)p(y|z)$, and explore candidate generators for obtaining samples from $q(\cdot)$. The CMI estimate is then obtained from the divergence estimator. 
\newline
\textbf{Difference Based CMI Estimation:} Using the improved MI estimates, and the difference relation $I(X;Y|Z) = I(X;YZ)-I(X;Z)$, we show that estimating CMI using a difference of two MI estimates performs best among several other proposed methods in this paper such as divergence based CMI estimation and KSG.
\newline
\textbf{Improved Performance in High Dimensions:} On both linear and non-linear data-sets, all our estimators perform significantly better than KSG. Surprisingly, our estimators perform well even for dimensions as high as $100$, while KSG fails to obtain reasonable estimates even beyond $5$ dimensions. 
\newline
\textbf{Improved Performance in Conditional Independence Testing:} As an application of CMI estimation, we use our best estimator for conditional independence testing (CIT) and obtain improved performance compared to the state-of-the-art CIT tester on both synthetic and real data-sets.
\vspace{-0.2in}
\section{Estimation of Conditional Mutual Information}
\vspace{-0.1in}
The CMI estimation problem from finite samples can be stated as follows. Let us consider three random variables $X$, $Y$, $Z \sim p(x,y,z)$, where $p(x,y,z)$ is the joint distribution. Let the dimensions of the random variables be $d_x$, $d_y$ and $d_z$ respectively. We are given $n$ samples $\{(x_i, y_i, z_i) \}_{i=1}^n$ drawn i.i.d from $p(x,y,z)$. So $x_i \in \mathbb{R}^{d_x}, y_i \in \mathbb{R}^{d_y}$ and $z_i \in \mathbb{R}^{d_z}$. The goal is to estimate $I(X;Y|Z)$ from these $n$ samples.
\vspace{-0.1in}
\subsection{Divergence Based CMI Estimation}
\vspace{-0.1in}
\begin{definition}
The Kullback-Leibler (KL) divergence between two distributions $p(\cdot)$ and $q(\cdot)$ is given as :
\[
D_{KL}(p||q) = \int p(x) \log \frac{p(x)}{q(x)} \, dx
\]
\end{definition}

\begin{definition}
Conditional Mutual Information (CMI) can be expressed as a KL-divergence between two distributions $p(x,y,z)$ and $q(x,y,z) = p(x,z)p(y|z)$, i.e., 
\[
I(X;Y|Z) = D_{KL}(p(x,y,z)||p(x,z)p(y|z))
\]
\end{definition}

The definition of CMI as a KL-divergence naturally leads to the question : \textit{Can we estimate CMI using an estimator for divergence ?} However, the problem is still non-trivial since we are only given samples from $p(x,y,z)$ and the divergence estimator would also require samples from $p(x,z)p(y|z)$. This further boils down to whether we can learn the distribution $p(y|z)$. 
\vspace{-0.1in}
\subsubsection{Generative Models}
\vspace{-0.05in}
We now explore various techniques to learn the conditional distribution $p(y|z)$ given samples $\sim p(x,y,z)$. This problem is fundamentally different from drawing independent samples from the marginals $p(x)$ and  $p(y)$, given the joint $p(x,y)$. In this simpler setting, we can simply permute the data to obtain $\{x_i, y_{\pi(i)} \}_{i=1}^n$ ($\pi$ denotes a permutation, $\pi(i) \neq i$). This would emulate samples drawn from $q(x,y) = p(x)p(y)$. But, such a permutation scheme does not work for $p(x,y,z)$ since it would destroy the dependence between $X$ and $Z$. The problem is solved using recent advances in generative models which aim to learn an unknown underlying distribution from samples.

\textbf{Conditional Generative Adversarial Network (CGAN)}: There exist extensions of the basic GAN framework \citep{goodfellow} in conditional settings, CGAN \citep{mirza}.
Once trained, the CGAN can then generate samples from the generator network as $y = \mathcal{G}(s, z), s \sim p(s), z \sim p(z)$. 


\textbf{Conditional Variational Autoencoder (CVAE)}: 
Similar to CGAN, the conditional setting, CVAE \citep{kingma} \citep{cvae}, aims to maximize the conditional log-likelihood. The input to the decoder network is the value of $z$ and the latent vector $s$ sampled from standard Gaussian. The decoder $Q$ gives the conditional mean and conditional variance (parametric functions of $s$ and $z$) from which $y$ is then sampled. 

\textbf{$k$NN based permutation}: A simpler algorithm for generating the conditional $p(y|z)$ is to permute data values where $z_i \approx z_j$. Such methods are popular in conditional independence testing literature \citep{ccit, kcit}. For a given point $\{x_i, y_i, z_i\}$, we find the $k$-nearest neighbor of $z_i$. Let us say it is $z_j$ with the corresponding data point as $\{x_j, y_j, z_j\}$. Then $\{x_i, y_j, z_i\}$ is a sample from $q(x,y,z)$.  

Now that we have outlined multiple techniques for estimating $p(y|z)$, we next proceed to the problem of estimating KL-divergence.

\subsubsection{Divergence Estimation}

Recently, \cite{belghazi} proposed a neural network based estimator of mutual information (MINE) by utilizing lower bounds on KL-divergence. Since MI is a special case of KL-divergence, their neural estimator can be extended for divergence estimation as well. The estimator can be trained using back-propagation and was shown to out-perform traditional methods for MI estimation. The core idea of MINE is cradled in a dual representation of KL-divergence. The two main lower bounds used by MINE are stated below.  

\begin{definition}
The Donsker-Varadhan representation expresses KL-divergence as a supremum over functions,
\begin{equation}
D_{KL}(p||q) = \sup\limits_{f \in \mathcal{F}} \mathop{\mathbb{E}}\limits_{x\sim p}[ f(x)] - \log(\mathop{\mathbb{E}}\limits_{x\sim q} [\exp(f(x))])
\label{DV-bound}
\end{equation}
\end{definition}
\vspace{-0.1in}

where the function class $\mathcal{F}$ includes those functions that lead to finite values of the expectations. 
\begin{definition}
The f-divergence bound gives a lower bound on the KL-divergence:
\begin{equation}
D_{KL}(p||q) \geq \sup\limits_{f \in \mathcal{F}} \mathop{\mathbb{E}}\limits_{x\sim p}[ f(x)] - \mathop{\mathbb{E}}\limits_{x\sim q} [\exp(f(x)-1)]
\label{f-bound}
\end{equation}
\end{definition}
\vspace{-0.1in}

MINE uses a neural network $f_{\theta}$ to represent the function class $\mathcal{F}$ and uses gradient descent to maximize the RHS in the above bounds.

Even though this framework is flexible and straight-forward to apply, it presents several practical limitations. The estimation is very sensitive to choices of hyper-parameters (hidden-units/layers) and training steps (batch size, learning rate). We found the optimization process to be unstable and to diverge at high dimensions (Section 4. Experimental Results). Our findings resonate those by \cite{poole} in which the authors found the networks difficult to tune even in toy problems. 

\vspace{-0.1in}
\subsection{Difference Based CMI Estimation}
\vspace{-0.1in}
Another seemingly simple approach to estimate CMI could be to express it as a difference of two mutual information terms by invoking the chain rule, i.e.: $I(X;Y|Z) = I(X; Y,Z) - I(X;Z)$.
As stated before, since mutual information is a special case of KL-divergence, viz. $I(X;Y) = D_{KL}(p(x,y)||p(x)p(y))$, this again calls for a stable, scalable, sample efficient KL-divergence estimator as we present in the next Section. 
 \vspace{-0.2in}
\section{Classifier Based MI Estimation} \label{CCMI}
\vspace{-0.1in}
In their seminal work on independence testing, \cite{lopezC2ST} introduced classifier two-sample test to distinguish between samples coming from two unknown distributions $p$ and $q$. The idea was also adopted for conditional independence testing by \cite{ccit}. The basic principle is to train a binary classifier by labeling samples $x \sim p$ as $1$ and those coming from $x \sim q$ as $0$, and to test the null hypothesis $\mathcal{H}_0 : p = q$. Under the null, the accuracy of the binary classifier will be close to $0.5$. It will be away from $0.5$ under the alternative. The accuracy of the binary classifier can then be carefully used to define $P$-values for the test. 

We propose to use the classier two-sample principle for estimating the likelihood ratio $\frac{p(x,y)}{p(x)p(y)}$. While existing literature has instances of using the likelihood ratio for MI estimation, the algorithms to estimate the likelihood ratio are quite different from ours. Both \citep{suzuki, nguyen2008} formulate the likelihood ratio estimation as a convex relaxation by leveraging the Legendre-Fenchel duality. But performance of the methods depend on the choice of suitable kernels and would suffer from the same disadvantages as mentioned in the Introduction. 
\vspace{-0.1in}
\subsection{Problem Formulation}
\vspace{-0.1in}
Given $n$ i.i.d samples $\{x_i^p\}_{i=1}^n , x_i^p \sim p(x)$ and $m$ i.i.d samples $\{x_j^q\}_{j=1}^m, x_j^q \sim q(x)$, we want to estimate $D_{KL}(p||q)$. We label the points drawn from $p(\cdot)$ as $y = 1$ and those from $q(\cdot)$ as $y=0$. A binary classifier is then trained on this supervised classification task. Let the prediction for a point $l$ by the classifier is $\gamma_l$ where $\gamma_l = Pr(y = 1 | x_l)$ ($Pr$ denotes probability). Then the point-wise likelihood ratio for data point $l$ is given by $\mathcal{L}(x_l) = \frac{\gamma_l}{1 -\gamma_l}$. 

The following Proposition is elementary and has already been observed in \cite{belghazi}(Proof of Theorem 4). We restate it here for completeness and quick reference. 

\begin{proposition} The optimal function in Donsker-Varadhan representation \eqref{DV-bound} is the one that computes the point-wise log-likelihood ratio, i.e, $f^*(x) =  \log \frac{p(x)}{q(x)} \, \forall \, x$, (assuming $p(x) = 0$, where-ever $q(x) = 0$).
\label{prop-1}
\end{proposition}

\vspace{-0.1in}
Based on Proposition \ref{prop-1}, the next step is to substitute the estimates of point-wise likelihood ratio in \eqref{DV-bound} to obtain an estimate of KL-divergence.
\vspace{-0.2in}
\begin{equation}
\hat{D}_{KL}(p||q) = \frac{1}{n}\sum\limits_{i=1}^n \log \mathcal{L}(x_i^p) - \log \left(  \frac{1}{m}\sum\limits_{j=1}^m \mathcal{L}(x_j^q) \right)
\label{plug-in}
\end{equation}

We obtain an estimate of mutual information from (\ref{plug-in}) as $\hat{I}_n(X;Y) = \hat{D}_{KL}(p(x,y)||p(x)p(y))$. This classifier-based estimator for MI (Classifier-MI) has the following theoretical properties under Assumptions (A1)-(A4) (stated in Section \ref{theo_ccmi}).
\begin{theorem}
Under Assumptions (A1)-(A4), Classifier-MI is consistent, i.e., given $\epsilon, \delta > 0, \exists \, n \in \mathbb{N} $, such that with probability at least $1-\delta$, we have \\
\vspace{-0.1in}
\[| \hat{I}_n (X;Y) - I(X;Y) | \leq \epsilon \]
\label{theorem-class-mi-consistent}
\end{theorem}
\vspace{-0.2in}

\begin{proof}
Here, we provide a sketch of the proof. The classifier is trained to minimize the binary cross entropy (BCE) loss on the train set and obtains the minimizer as $\hat{\theta}$. From generalization bound of classifier, the loss value on the test set from $\hat{\theta}$ is close to the loss obtained by the best optimizer in the classifier family, which itself is close to the global minimizer $\gamma*$ of BCE (as a function $\gamma$) by Universal Function Approximation Theorem of neural-networks. 

The $\mathrm{BCE}$ loss is strongly convex in $\gamma$. $\gamma$ links $\mathrm{BCE}$ to $I(\cdot \,;\, \cdot)$, i.e., $|\mathrm{BCE}_n(\gamma_{\hat{\theta}}) - \mathrm{BCE}(\gamma^*)| \leq \epsilon' \implies \| \gamma_{\hat{\theta}} - \gamma^*\|_1 \leq \eta \implies | \hat{I}_n (X;Y) - I(X;Y)| \leq \epsilon$. 
\end{proof}

While consistency provides a characterization of the estimator in large sample regime, it is not clear what guarantees we obtain for finite samples. The following Theorem shows that even for a small number of samples, the produced MI estimate is a true lower bound on mutual information value with high probability.

\begin{theorem}
Under Assumptions (A1)-(A4), the finite sample estimate from Classifier-MI is a lower bound on the true MI value with high probability, i.e., given $n$ test samples, we have for $\epsilon > 0$
\vspace{-0.1in}
\[
Pr(I(X;Y) + \epsilon \geq \hat{I}_n (X;Y) ) \geq 1 - 2\exp(-Cn)
\]
where $C$ is some constant independent of $n$ and the dimension of the data.
\label{theorem-true-lower-bound}
\end{theorem}

\vspace{-0.2in}
\subsection{Probability Calibration}
\vspace{-0.1in}
The estimation of likelihood ratio from classifier predictions $Pr(y=1|x)$ hinges on the fact that the classifier is well-calibrated. As a rule of thumb, classifiers trained directly on the cross entropy loss are well-calibrated. But boosted decision trees would introduce distortions in the likelihood-ratio estimates. There is an extensive literature devoted to obtaining better calibrated classifiers that can be used to improve the estimation further \citep{lakshminarayanan, niculescu, guo2017calibration}. We experimented with Gradient Boosted Decision Trees and multi-layer perceptron trained on the log-loss in our algorithms. Multi-layer perceptron gave better estimates and so is used in all the experiments. Supplementary Figures show that the neural networks used in our estimators are well-calibrated.

Even though logistic regression is well-calibrated and might seem to be an attractive candidate for classification in sparse sample regimes, we show that linear classifiers cannot be used to estimate $D_{KL}$ by two-sample approach. For this, we consider the simple setting of estimating mutual information of two correlated Gaussian random variables as a counter-example. 

\begin{figure*}[t]
\centering
\begin{subfigure}[b]{0.47\textwidth}
\includegraphics[width=\textwidth]{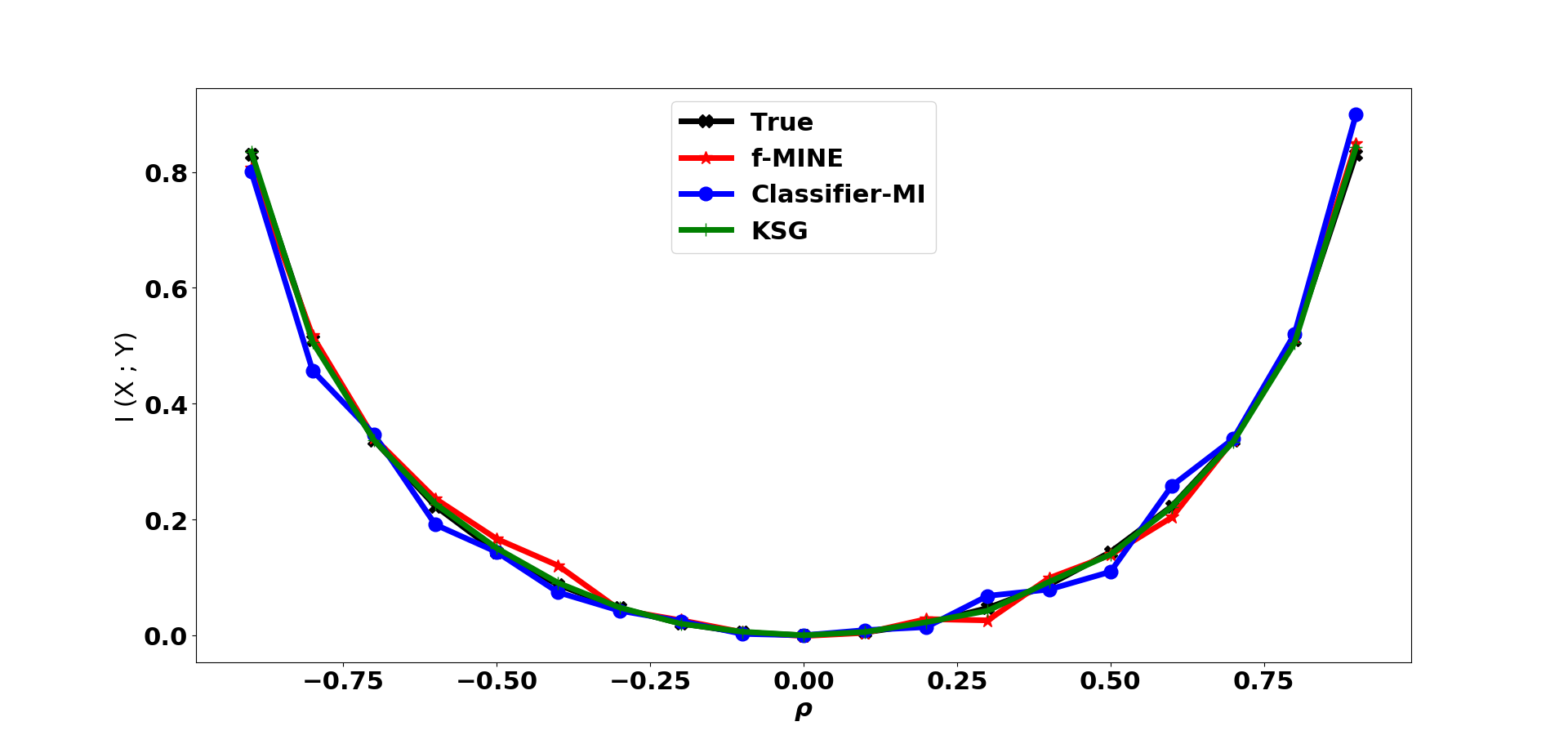} 
\caption{$d_x = d_y = 1$}
\end{subfigure}
\begin{subfigure}[b]{0.47\textwidth}
\includegraphics[width=\textwidth]{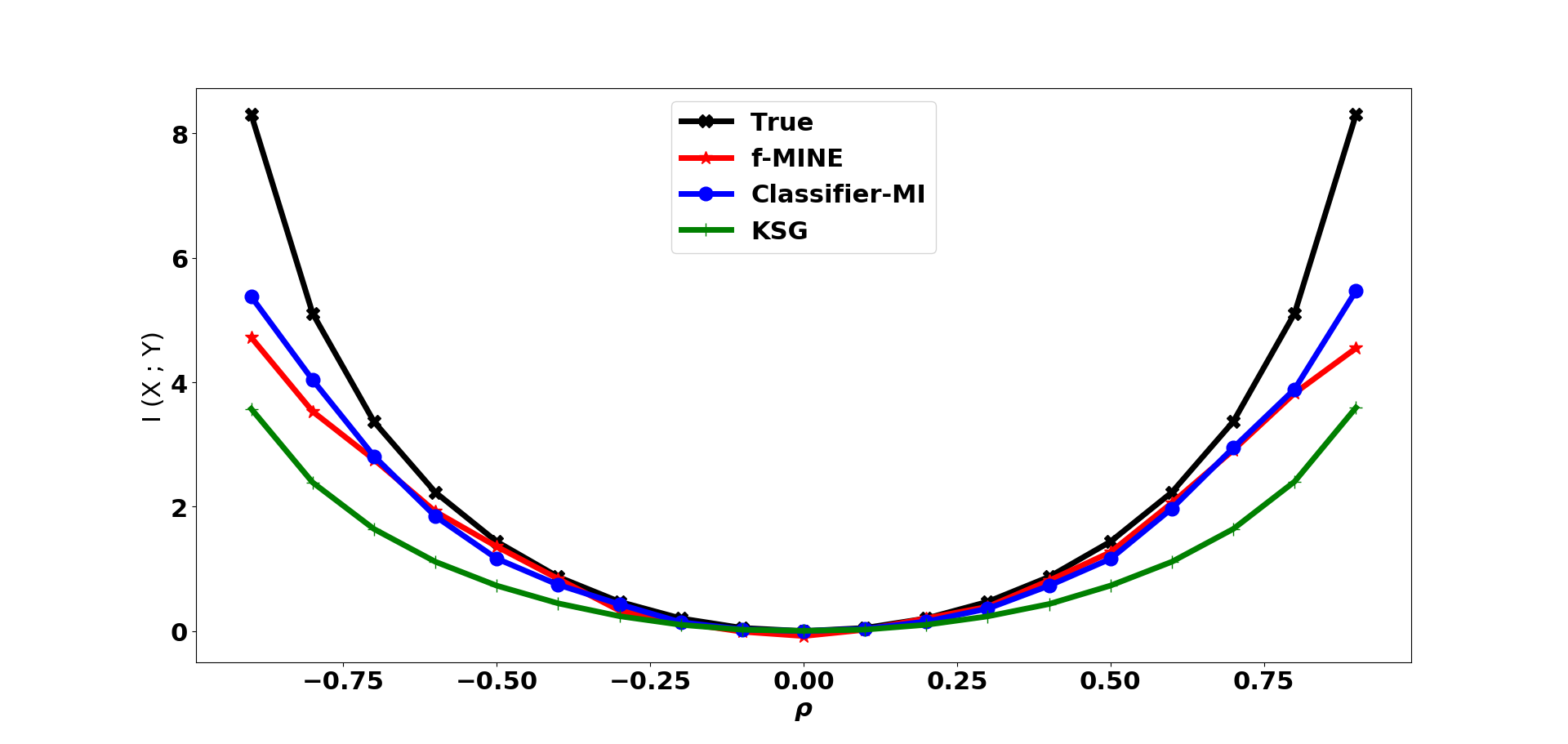} 
\caption{$d_x = d_y = 10$}
\end{subfigure}
\caption{Mutual Information Estimation of Correlated Gaussians : In this setting, $X$ and $Y$ have independent co-ordinates, with ($X_i, Y_i$) $\forall \,i$  being correlated Gaussians with correlation coefficient $\rho$.$\,\, I^*(X;Y) = -\frac{1}{2}d_x \log(1-\rho^2)$}
\label{toy}
\vspace{-0.2in}
\end{figure*}

\begin{lemma}
A linear classifier with marginal features fails the classifier Two sample MI estimation.
\label{lemma-lin}
\end{lemma}

\vspace{-0.2in}
\begin{proof}

Consider two correlated Gaussians in $2$ dimensions $(X_1, X_2) \sim \mathcal{N} \left( 0, M =\bigl( \begin{smallmatrix}1 & \rho\\ \rho & 1\end{smallmatrix}\bigr)  \right)$, where $\rho$ is the Pearson correlation. The marginals are standard Gaussans $X_i \sim \mathcal{N}(0, 1)$. Suppose we are trying to estimate the mutual information $D_{KL}(p(x_1, x_2)||p(x_1)p(x_2))$. The classifier decision boundary would seek to find 
$Pr(y=1|x_1, x_2) > Pr(y=0|x_1, x_2)$, thus
$ p(x_1, x_2) > p(x_1)p(x_2) => x_1x_2 >  \frac{1}{2\rho} \log(1-\rho^2)$
\end{proof}
\vspace{-0.1in}
The decision boundary is a rectangular hyperbola. Here the classifier would return $0.5$ as prediction for either class (leading to $\hat{D}_{KL} = 0$), even when $X_1$ and $X_2$ are highly correlated and the mutual information is high. 

We use the Classifier two-sample estimator to first compute the mutual information of two correlated Gaussians \citep{belghazi} for $n = 5,000$ samples. This setting also provides us a way to choose reasonable hyper-parameters that are used throughout in all the synthetic experiments. We also plot the estimates of f-MINE and KSG to ensure we are able to make them work in simple settings. In the toy setting $d_x = 1$, all estimators accurately estimate $I(X;Y)$ as shown in Figure \ref{toy}.
\vspace{-0.1in}
\subsection{Modular Approach to CMI Estimation}
\vspace{-0.1in}
Our classifier based divergence estimator does not encounter an optimization problem involving exponentials. MINE optimizing \eqref{DV-bound} has biased gradients while that based on \eqref{f-bound} is a weaker lower bound \citep{belghazi}. On the contrary, our classifier is trained on cross-entropy loss which has unbiased gradients. Furthermore, we plug in the likelihood ratio estimates into the tighter Donsker-Varadhan bound, thereby, achieving the best of both worlds. Equipped with a KL-divergence estimator, we can now couple it with the generators or use the expression of CMI as a difference of two MIs (which we address from now as MI-Diff.). Algorithm \ref{alg:ccmi} describes the CMI estimation by tying together the generator and Classifier block. For MI-Diff., function block ``Classifier-$D_{KL}$'' in Algorithm \ref{alg:ccmi} has to be used twice : once for estimating $I(X; Y,Z)$ and another for $I(X;Z)$. For mutual information, $\mathcal{D}_q$ in ``Classifier-$D_{KL}$'' is obtained by permuting the samples of $p(\cdot)$.


For the Classifier coupled with a generator, the generated distribution $g(y|z)$ may deviate from the target distribution $p(y|z)$ - introducing a different kind of bias. The following Lemma suggests how such a bias can be corrected by subtracting the KL divergence of the sub-tuple $(Y,Z)$ from the divergence of the entire triple $(X,Y,Z)$. We note that such a clean relationship is not true for general divergence measures, and indeed require more sophisticated conditions for the total-variation metric \citep{sen2018mimic}.
\vspace{-0.05in}
\begin{lemma}[Bias Cancellation]
The estimation error due to incorrect generated distribution $g(y|z)$ can be accounted for using the following relation :
\vspace{-0.1in}
\begin{align*}
&D_{KL}(p(x,y,z)||p(x,z)p(y|z)) = \\
&D_{KL}(p(x,y,z)||p(x,z)g(y|z)) - D_{KL}(p(y,z)||p(z)g(y|z))
\end{align*}
\label{Lemma-2}
\end{lemma}
\vspace{-0.05in}

\begin{algorithm*}
\DontPrintSemicolon
\SetAlgoLined
\KwIn{Dataset $\mathcal{D} = \{x_i, y_i, z_i \}_{i=1}^n$, number of outer boot-strap iterations $B$, Inner iterations $T$, clipping constant $\tau$.}
\KwOut{CMI estimatate $\hat{I}(X;Y|Z)$}
\For{$b \in \{ 1, 2, \ldots B \}$}{
Permute the points in dataset $\mathcal{D}$ to obtain $\mathcal{D}^\pi$. \;
Split $\mathcal{D}^\pi$ equally into two parts $\mathcal{D}_\mathrm{class, joint} = \{x_i, y_i, z_i\}_{i=1}^{n/2}$ and $\mathcal{D}_\mathrm{gen} = \{x_i, y_i, z_i\}_{i=n/2}^{n}$. \;
Train the generator $\mathcal{G}(\cdot)$ on $\mathcal{D}_\mathrm{gen}$. \;
Generate the marginal data-set using points $y'_i = \mathcal{G}(z_i) \, \forall\, z_i \in \mathcal{D}_\mathrm{class, joint}(:,Z)$. $\mathcal{D}_\mathrm{class, marg} = \{x_i, y'_i, z_i\}_{i=1}^{n/2}$ \;
    $\hat{I}_b(X;Y|Z) = \mathrm{Classifier\_D_{KL}}(\mathcal{D}_\mathrm{class, joint}, \mathcal{D}_\mathrm{class, marg}, T, \tau)$
}
\Return{$\frac{1}{B} \sum_b \hat{I}_b(X;Y|Z)$}\;
\SetKwFunction{FMain}{$\mathrm{Classifier\_D_{KL}}$}
  \SetKwProg{Fn}{Function}{:}{}
  \Fn{\FMain{$\mathcal{D}_p, \mathcal{D}_q, T, \tau$}}{
       Label points $u \in \mathcal{D}_p$ as $l=1$ and $v \in \mathcal{D}_q$ as $l=0$. \;
\For{$t \in \{ 1, 2, \ldots T \}$}{
    $\mathcal{D}_p^{\mathrm{train}}, \mathcal{D}_p^{\mathrm{eval}} \leftarrow$ {\sc split\_test\_train}($\mathcal{D}_p$). \;
     $\mathcal{D}_q^{\mathrm{train}}, \mathcal{D}_q^{\mathrm{eval}} \leftarrow$ {\sc split\_test\_train}($\mathcal{D}_q$)  \;
    Train classifier $\mathcal{C}$ on $\{ \mathcal{D}_p^{\mathrm{train}} ,\vec{1} \}, \{ \mathcal{D}_q^{\mathrm{train}} ,\vec{0} \}$\;
    Obtain classifier predictions $Pr(l=1|w) \, \forall \,w \in \mathcal{D}_p^{\mathrm{eval}} \cup \mathcal{D}_q^{\mathrm{eval}}$, and clip to $[\tau, 1-\tau]$.\;
    $\hat{D}_{KL}^t(p || q) \leftarrow \frac{1}{|\mathcal{D}_p^{\mathrm{eval}}|} \sum\limits_{u \in \mathcal{D}_p^{\mathrm{eval}}} \log \frac{Pr(l=1|u)}{1-Pr(l=1|u)} - \log \left( \frac{1}{|\mathcal{D}_q^{\mathrm{eval}}|} \sum\limits_{v \in \mathcal{D}_q^{\mathrm{eval}}}  \frac{Pr(l=1|v)}{1-Pr(l=1|v)}   \right)$  \;    
    }
        \KwRet $\hat{D}_{KL}(p || q)  = \frac{1}{\tau} \sum_t \hat{D}_{KL}^t(p || q) $ \;
  }
\caption{{\sc Generator + Classifier}}
\label{alg:ccmi}
\end{algorithm*}

\vspace{-0.4in}

\section{Experimental Results}
\vspace{-0.1in}
In this Section, we compare the performance of various estimators on the CMI estimation task. We used the Classifier based divergence estimator (Section 3) and MINE in our experiments. \cite{belghazi} had two MINE variants, namely Donsker-varadhan (DV) MINE and f-MINE. The f-MINE has unbiased gradients and we found it to have similar performance as DV-MINE, albeit with lower variance. So we used f-MINE in all our experiments. 

The ``generator''+``Divergence estimator'' notation will be used to denote the various estimators. For instance, if we use CVAE for the generation and couple it with f-MINE, we denote the estimator as CVAE+f-MINE. When coupled with the Classifier based Divergence block, it will be denoted as CVAE+Classifier. For MI-Diff. we represent it similarly as MI-Diff.+``Divergence estimator''. 

We compare our estimators with the widely used KSG estimator.\footnote{The implementation of CMI estimator in Non-parametric Entropy Estimation Toolbox (\url{https://github.com/gregversteeg/NPEET}) is used.} For f-MINE, we used the code provided to us by the author \citep{belghazi}. The same hyper-parameter setting is used in all our synthetic data-sets for all estimators (including generators and divergence blocks). Supplementary contains the details about the hyper-parameter values. For KSG, we vary $k \in \{3, 5, 10\}$ and report the results for the best $k$ for each data-set.
\vspace{-0.1in}
\subsection{Linear Relations}
\vspace{-0.1in}
\begin{figure*}[t]
\centering
\begin{subfigure}[b]{0.47\textwidth}
\includegraphics[width=\textwidth]{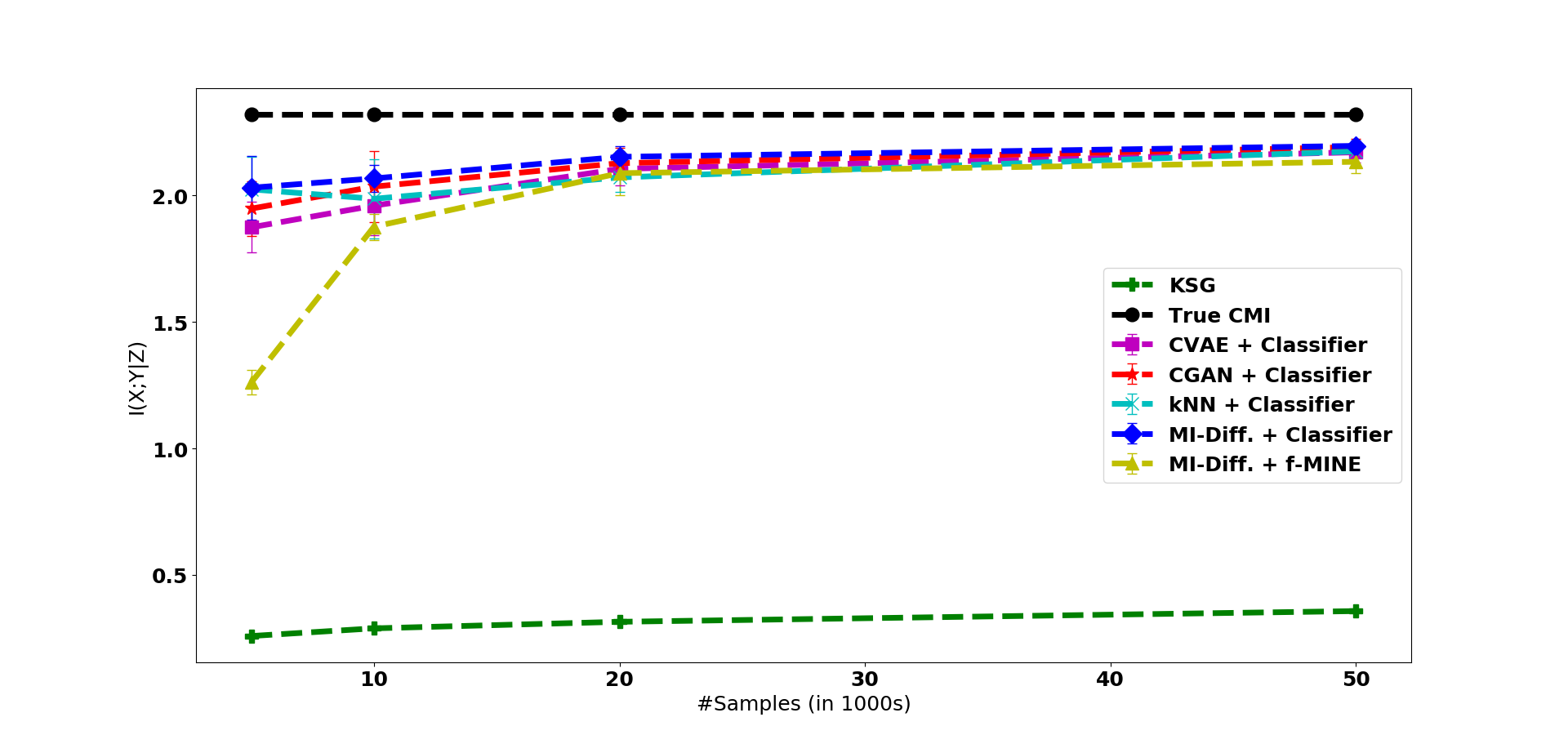}
\caption{Model I : Variation with $n$, $d_z = 20$}
\end{subfigure}
\begin{subfigure}[b]{0.47\textwidth}
\includegraphics[width=\textwidth]{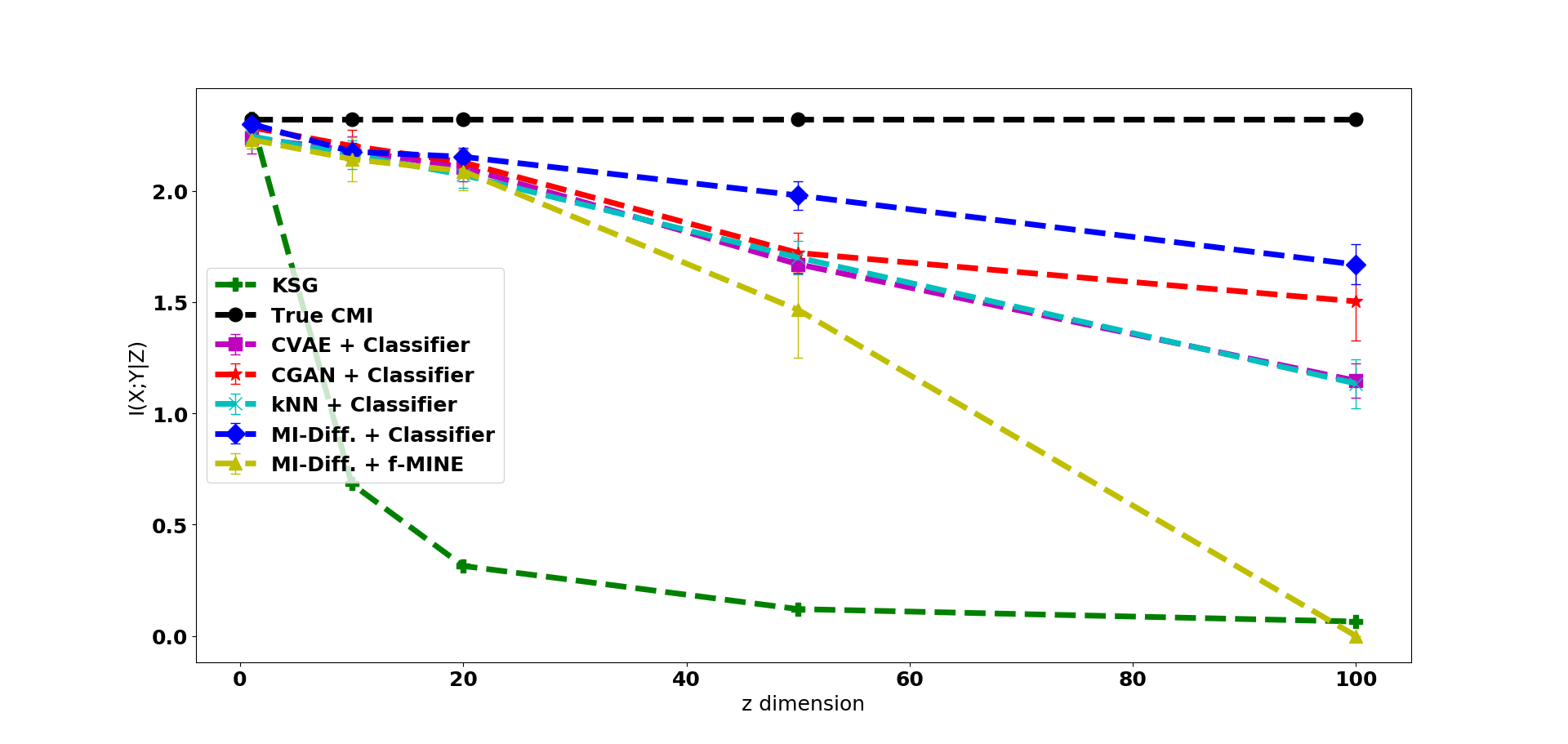}
\caption{Model I : Variation with $d_z$, $n = 20,000$}
\end{subfigure}
\begin{subfigure}[b]{0.47\textwidth}
\includegraphics[width=\textwidth]{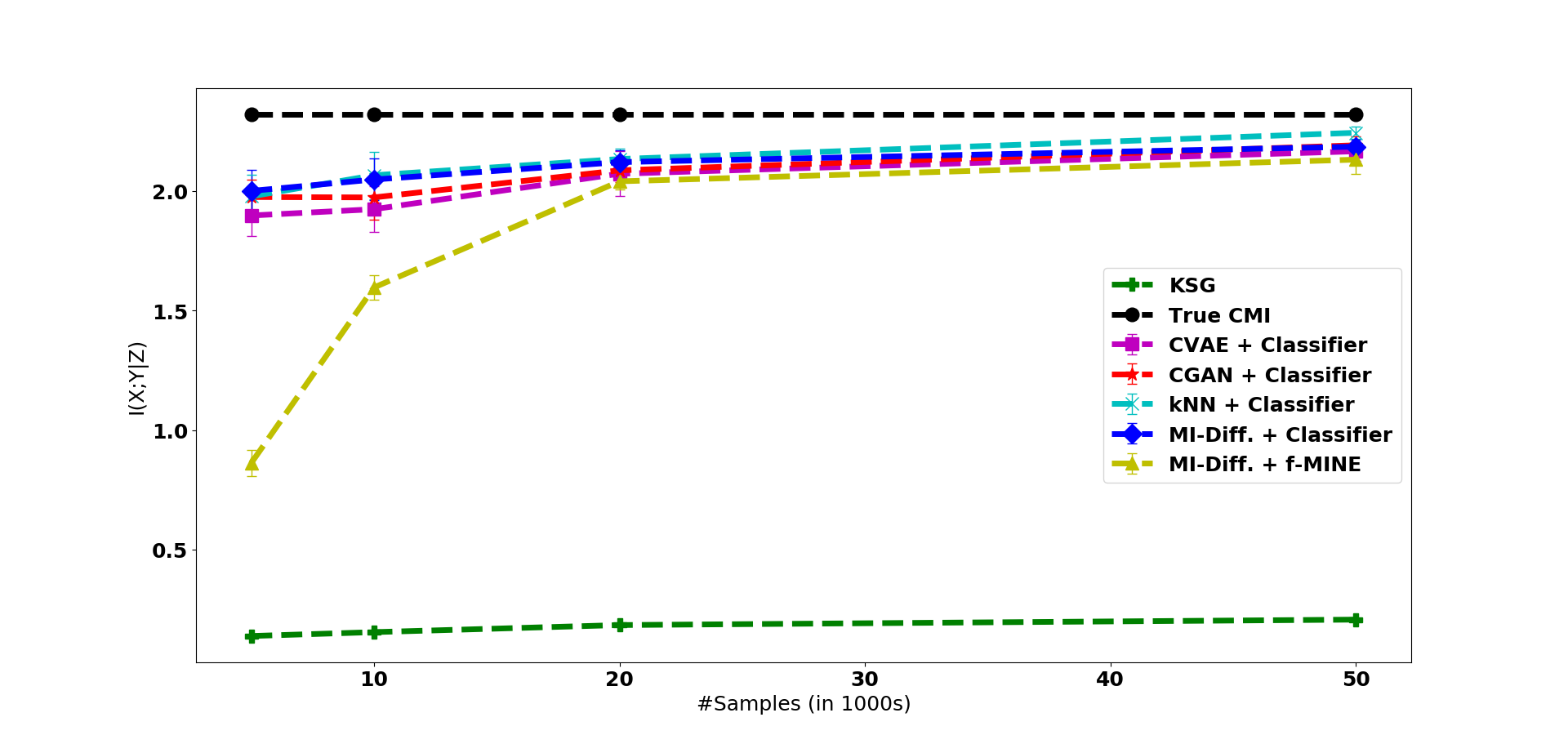}
\caption{Model II : Variation with $n$, $d_z = 20$}
\end{subfigure}
\begin{subfigure}[b]{0.47\textwidth}
\includegraphics[width=\textwidth]{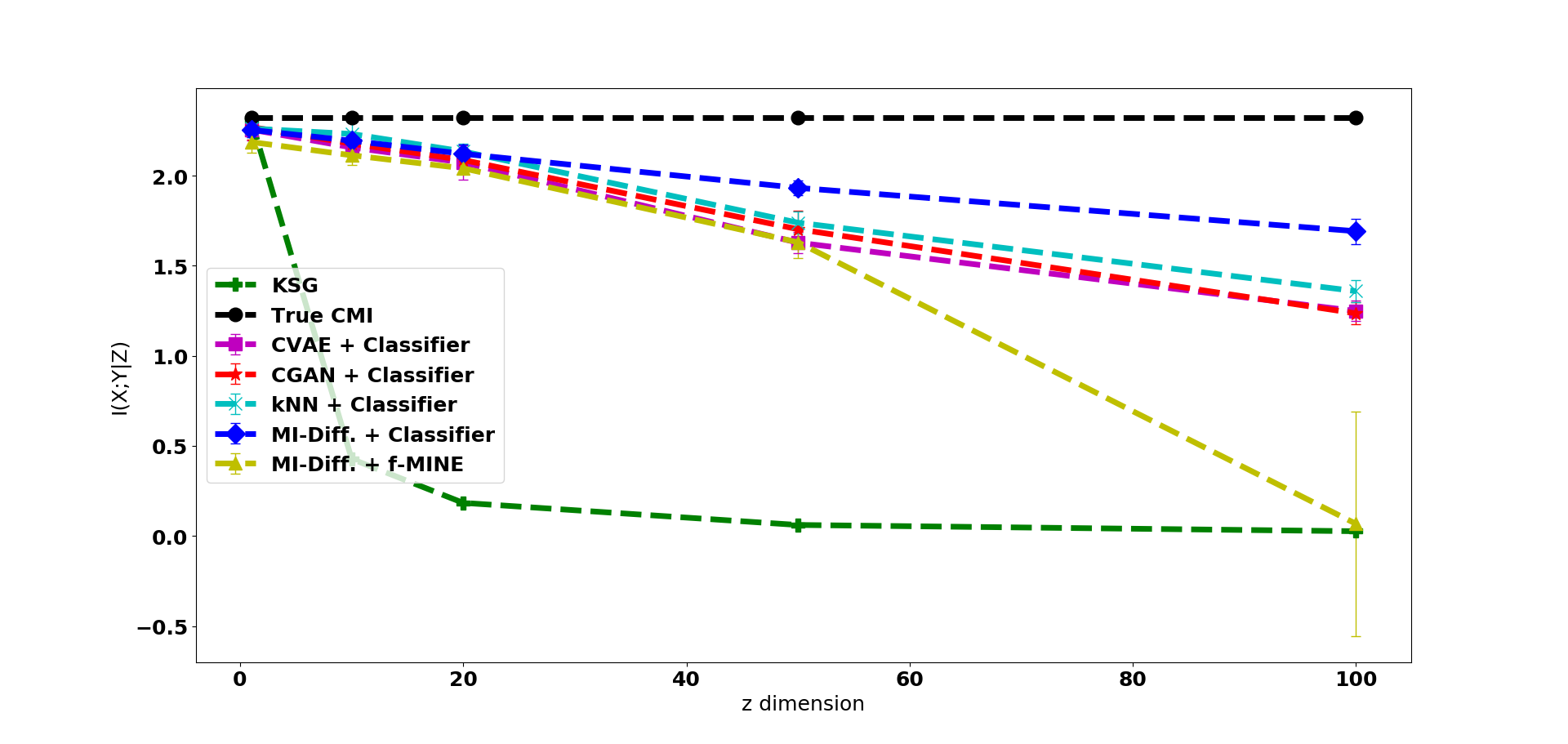}
\caption{Model II : Variation with $d_z$, $n = 20,000$}
\end{subfigure}
\caption{CMI Estimation in Linear models : We study the effect of various estimators as either number of samples $n$ or dimension $d_z$ is varied. MI-Diff.+Classifier performs the best among our estimators, while all our proposed estimators improve the estimation significantly over KSG. Average of $10$ runs is plotted. Error bars depict $1$ standard deviation from mean. (Best viewed in color)}
\label{lin-syn}
\end{figure*} 

We start with the simple setting where the three random variables $X$, $Y$, $Z$ are related in a linear fashion. We consider the following two linear models. 
\vspace{-0.1in}
\begin{table}[h]
\caption{Linear Models}
\label{lin-model}
\begin{center}
\begin{tabular}{ll}
\multicolumn{1}{c}{\bf Model I}  &\multicolumn{1}{c}{\bf Model II} \\
\hline \\
$X \sim \mathcal{N}(0,1)$          &$X \sim \mathcal{N}(0,1)$ \\
$Z \sim \mathcal{U}(-0.5, 0.5)^{d_z}$              &$Z \sim \mathcal{N}(0, 1)^{d_z}$ \\
&$U = w^TZ, \|w\|_1 = 1$  \\
$\epsilon \sim \mathcal{N}(Z_1, \sigma_\epsilon^2) $             &$\epsilon \sim \mathcal{N}(U, \sigma_\epsilon^2)$  \\
$Y \sim X + \epsilon$             &$Y \sim X + \epsilon$  \\
\end{tabular}
\end{center}
\end{table}

where $\mathcal{U}(-0.5, 0.5)^{d_z}$ means that each co-ordinate of $Z$ is drawn i.i.d from a uniform distribution between $-0.5$ and $0.5$. Similar notation is used for the Gaussian : $\mathcal{N}(0, 1)^{d_z}$. $Z_1$ is the first dimension of $Z$. We used $\sigma_\epsilon = 0.1$ and obtained the constant unit norm random vector $w$ from $\mathcal{N}(0,I_{d_z})$. $w$ is kept constant for all points during data-set preparation.

As common in literature on causal discovery and independence testing \citep{ccit, kcit}, the dimension of $X$ and $Y$ is kept as $1$, while $d_z$ can scale. Our estimators are general enough to accommodate multi-dimensional $X$ and $Y$, where we consider a concatenated vector $X = (X_1, X_2,\ldots, X_{d_x})$ and $Y = (Y_1, Y_2,\ldots, Y_{d_y})$. This has applications in learning interactions between Modules in Bayesian networks \citep{segal} or dependence between group variables \citep{entner, parviainen} such as distinct functional groups of proteins/genes instead of individual entities. Both the linear models are representative of problems encountered in Graphical models and independence testing literature. In Model I, the conditioning set can go on increasing with independent variables $\{Z_k\}_{k=2}^{d_z}$, while $Y$ only depends on $Z_1$. In Model II, we have the variables in the conditioning set combining linearly to produce $Y$. It is also easy to obtain the ground truth CMI value in such models by numerical integration.

For both these models, we generate data-sets with varying number of samples $n$ and varying dimension $d_z$ to study their effect on estimator performance. The sample size is varied as $n \in \{5000, 10000, 20000, 50000\}$ keeping $d_z$ fixed at $20$. We also vary $d_z \in \{1, 10, 20, 50, 100\}$, keeping sample size fixed at $n = 20000$.

Several observations stand out from the experiments: (1) KSG estimates are accurate at very low dimension but drastically fall with increasing $d_z$ even when the conditioning variables are completely independent and do not influence $X$ and $Y$ (Model-I).
(2) Increasing the sample size does not improve KSG estimates once the dimension is kept moderate (even $20$!). The dimension issue is more acute than sample scarcity.
(3) The estimates from f-MINE have greater deviation from the truth at low sample sizes. At high dimensions, the instability is clearly portrayed when the estimate suddenly goes negative (Truncated to $0.0$ to maintain the scale of the plot). (4) All our estimators using Classifier are able to obtain reasonable estimates even at dimensions as high as $100$, with MI-Diff.+Classifier performing the best.

\vspace{-0.1in}
\subsection{Non-Linear Relations}
\vspace{-0.1in}
\begin{figure*}[t]
\centering
\begin{subfigure}[b]{0.47\textwidth}
\includegraphics[width=\textwidth]{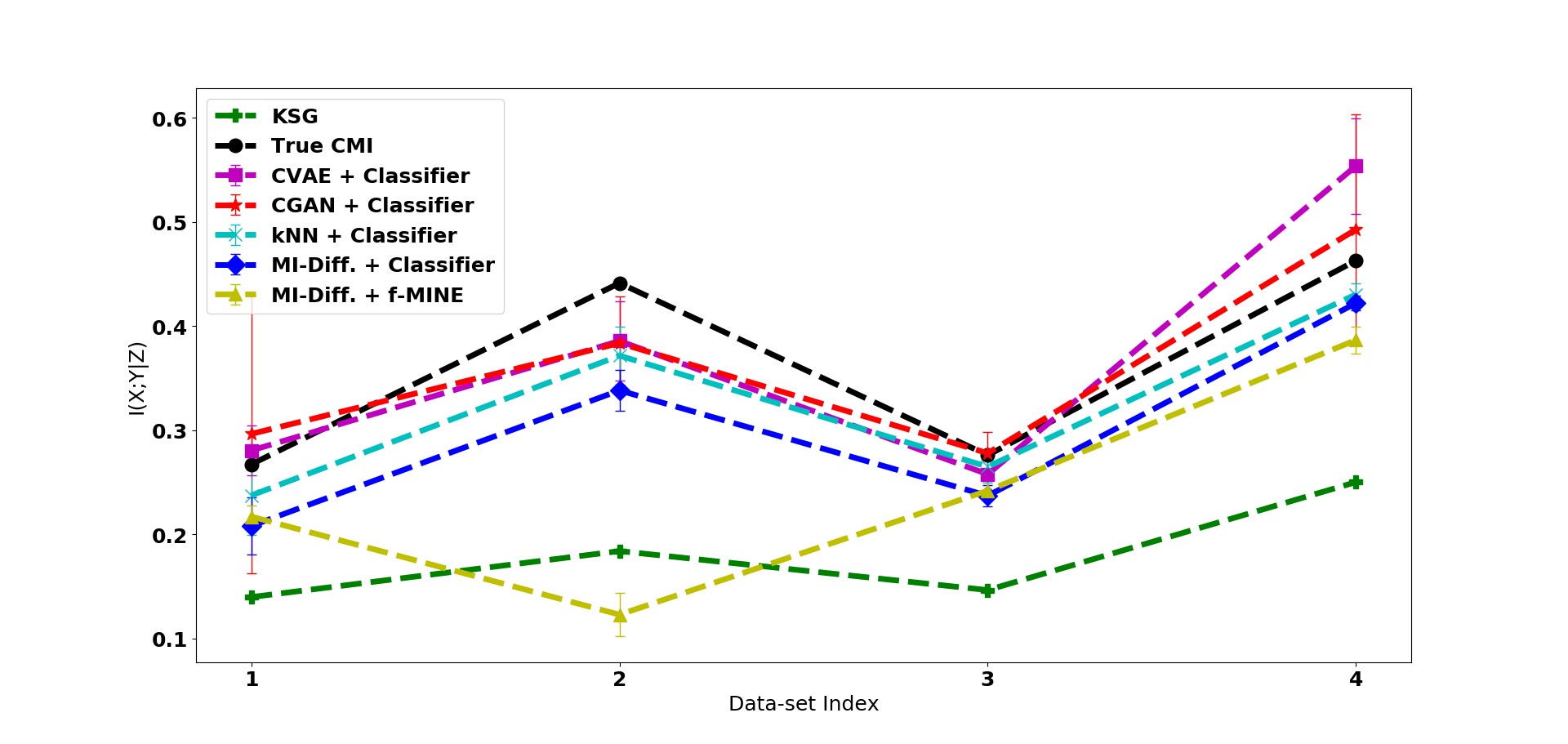} 
\caption{Non-linear Model : Number of samples increase with \\Data-index, $d_z = 10$ (fixed)}
\end{subfigure}
\begin{subfigure}[b]{0.47\textwidth}
\includegraphics[width=\textwidth]{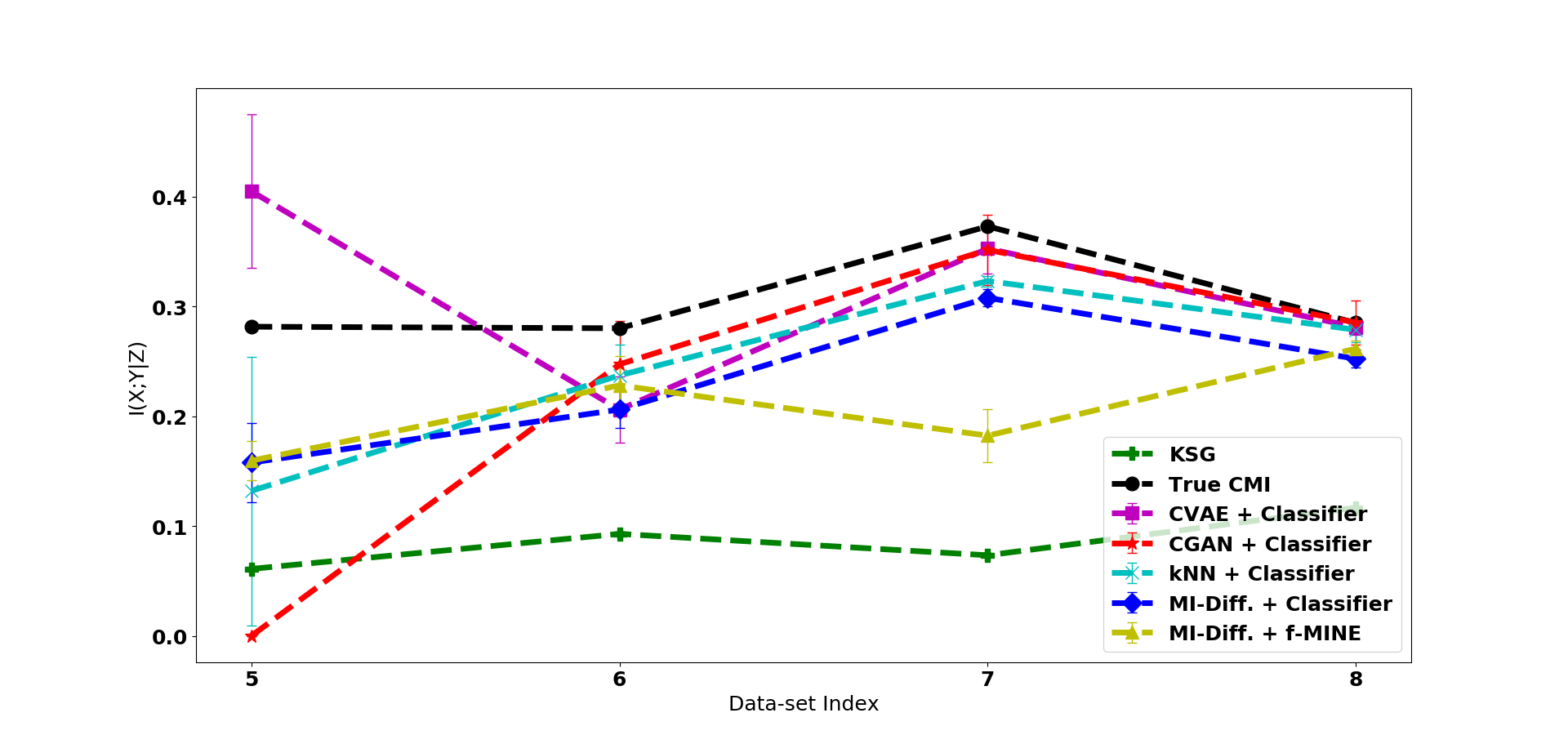}
\caption{Non-linear Model : Number of samples increase with \\Data-index, $d_z = 20$ (fixed)}
\end{subfigure}
\begin{subfigure}[b]{\textwidth}
\center 
\includegraphics[width=0.7\textwidth]{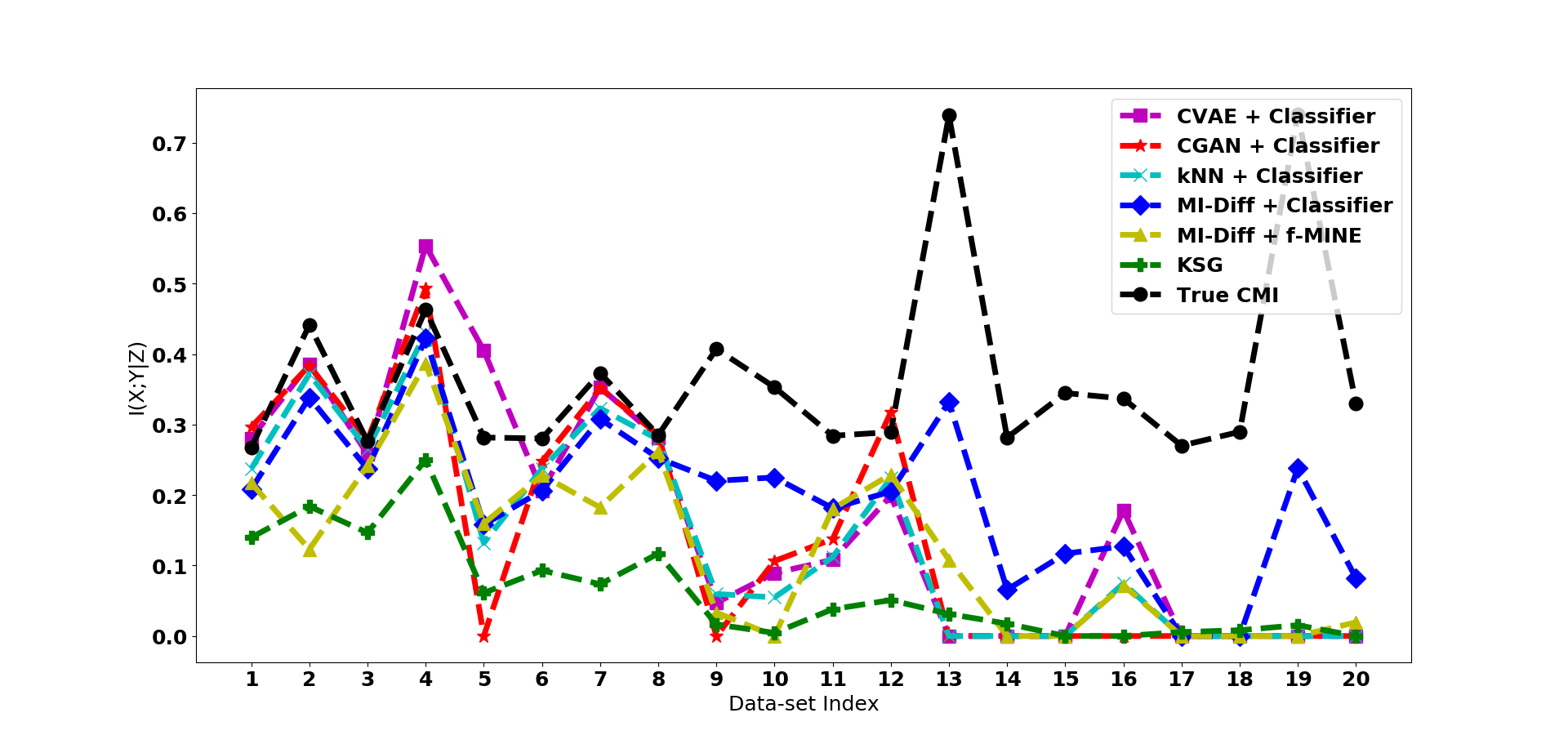}
\caption{Non-linear Models (All $20$ data-sets)}
\end{subfigure}
\caption{On non-linear data-sets, a similar trend is observed. KSG under-estimates $I^*(X;Y|Z)$, while our estimators track it closely. Average over $10$ runs is plotted. (Best viewed in color)}
\label{non-lin-syn}
\end{figure*}

Here, we study models where the underlying relations between $X$, $Y$ and $Z$ are non-linear. Let $Z \sim \mathcal{N}(\mathbbm{1}, I_{d_z}), X = f_1 (\eta_1), Y = f_2(A_{zy} Z + A_{xy} X + \eta_2)$. $f_1$ and $f_2$ are non-linear bounded functions drawn uniformly at random from $\{cos(\cdot), tanh(\cdot), \exp(-|\cdot|)\}$ for each data-set. $A_{zy}$ is a random vector whose entries are drawn $\mathcal{N}(0,1)$ and normalized to have unit norm. The vector once generated is kept fixed for a particular data-set. We have the setting where $d_x = d_y = 1$ and $d_z$ can scale. $A_{xy}$ is then a constant. We used $A_{xy} = 2$ in our simulations. The noise variables $\eta_1, \eta_2$ are drawn i.i.d $\mathcal{N}(0,\sigma_\epsilon^2)$, $\sigma_\epsilon^2 = 0.1$.

We vary $n \in \{5000, 10000, 20000, 50000 \}$ across each dimension $d_z$. The dimension $d_z$ itself is then varied as $\{10, 20, 50, 100, 200\}$ giving rise to $20$ data-sets. Data-index $1$ has $n = 5000, d_z = 10$, data-index $2$ has $n = 10000, d_z = 10$ and so on until data-index $20$ with $n = 50000, d_z = 200$.

\textbf{Obtaining Ground Truth $I^*(X;Y|Z)$} : Since it is not possible to obtain the ground truth CMI value in such complicated settings using a closed form expression, we resort to using the relation $I(X;Y|Z) = I(X;Y|U)$ where $U = A_{zy} Z$. The dependence of $Y$ on $Z$ can be completely captured once $U$ is given. But, $U$ has dimension $1$ and can be estimated accurately using KSG. We generate $50000$ samples separately for each data-set to estimate $I(X;Y|U)$ and use it as the ground truth.

\begin{figure*}[t]
\centering
\begin{subfigure}[b]{0.47\textwidth}
\includegraphics[width=\textwidth]{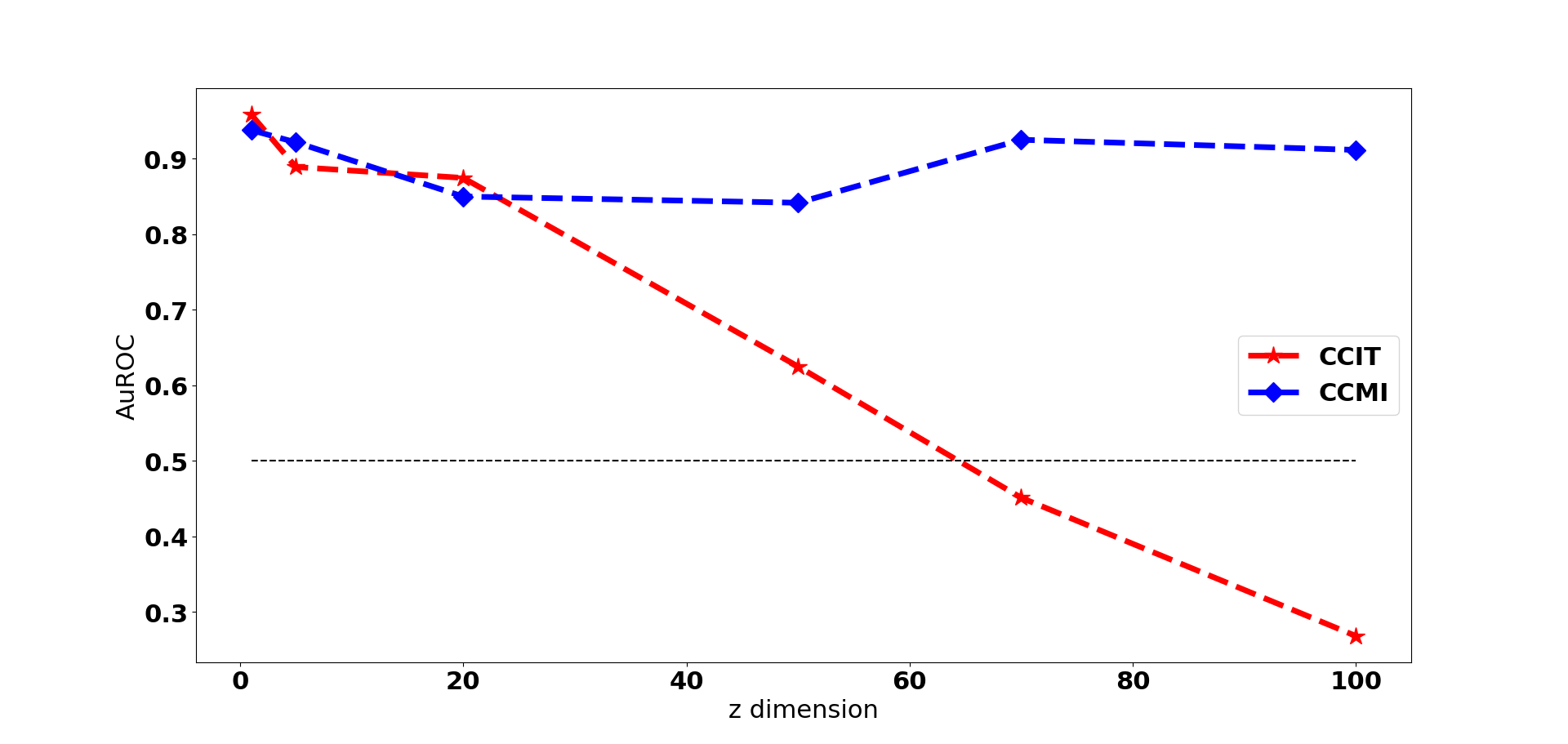}
\caption{CCIT performance degrades with increasing dz; \\ CCMI
retains high AuROC score even at dz = 100.}
\end{subfigure}
\begin{subfigure}[b]{0.47\textwidth}
\includegraphics[width=\textwidth]{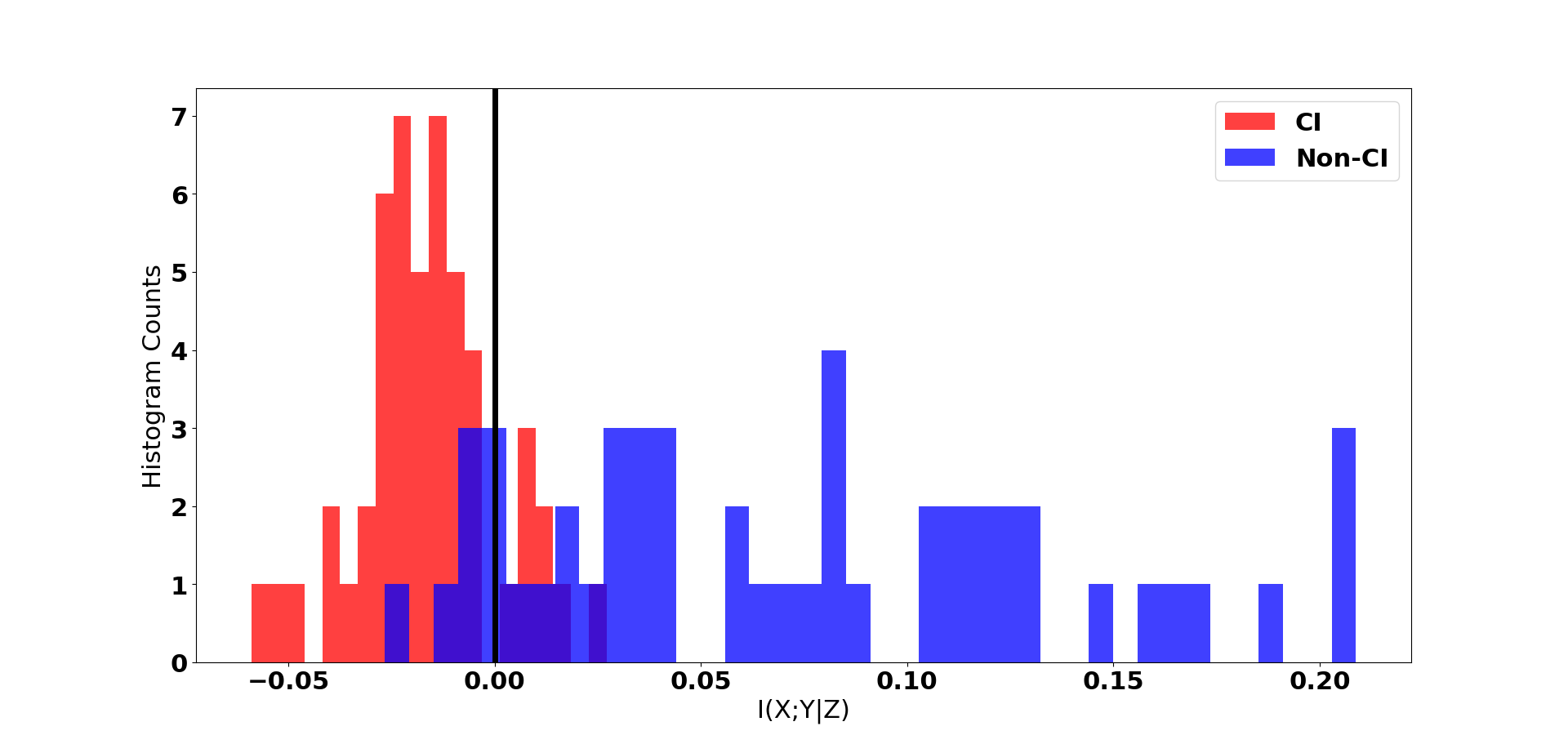} 
\caption{Estimates for CI data-sets are $\leq0$ and those for non-CI are $>0$ at $d_z = 100$. Thresholding CMI estimates at $0$ yields Precision = $0.84$, Recall = $0.86$.} 
\label{cmi-stat}
\end{subfigure}
\caption{Conditional Independence Testing in Post Non-linear Synthetic Data-set}
\label{postNonLin-CIT}
\vspace{-0.2in}
\end{figure*}

We observed similar behavior (as in Linear models) for our estimators in the Non-linear setting. 

(1) KSG continues to have low estimates even though in this setup the true CMI values are themselves low ($<1.0$).
(2) Up to $d_z = 20$, we find all our estimators closely tracking $I^*(X;Y|Z)$. But in higher dimensions, they fail to perform accurately. 
(3) MI-Diff. + Classifier is again the best estimator, giving CMI estimates away from $0$ even at $200$ dimensions. 

From the above experiments, we found MI-Diff.+Classifier to be the most accurate and stable estimator. We use this combination for our downstream applications and henceforth refer to it as CCMI. 

\begin{figure}
\includegraphics[width=0.47\textwidth]{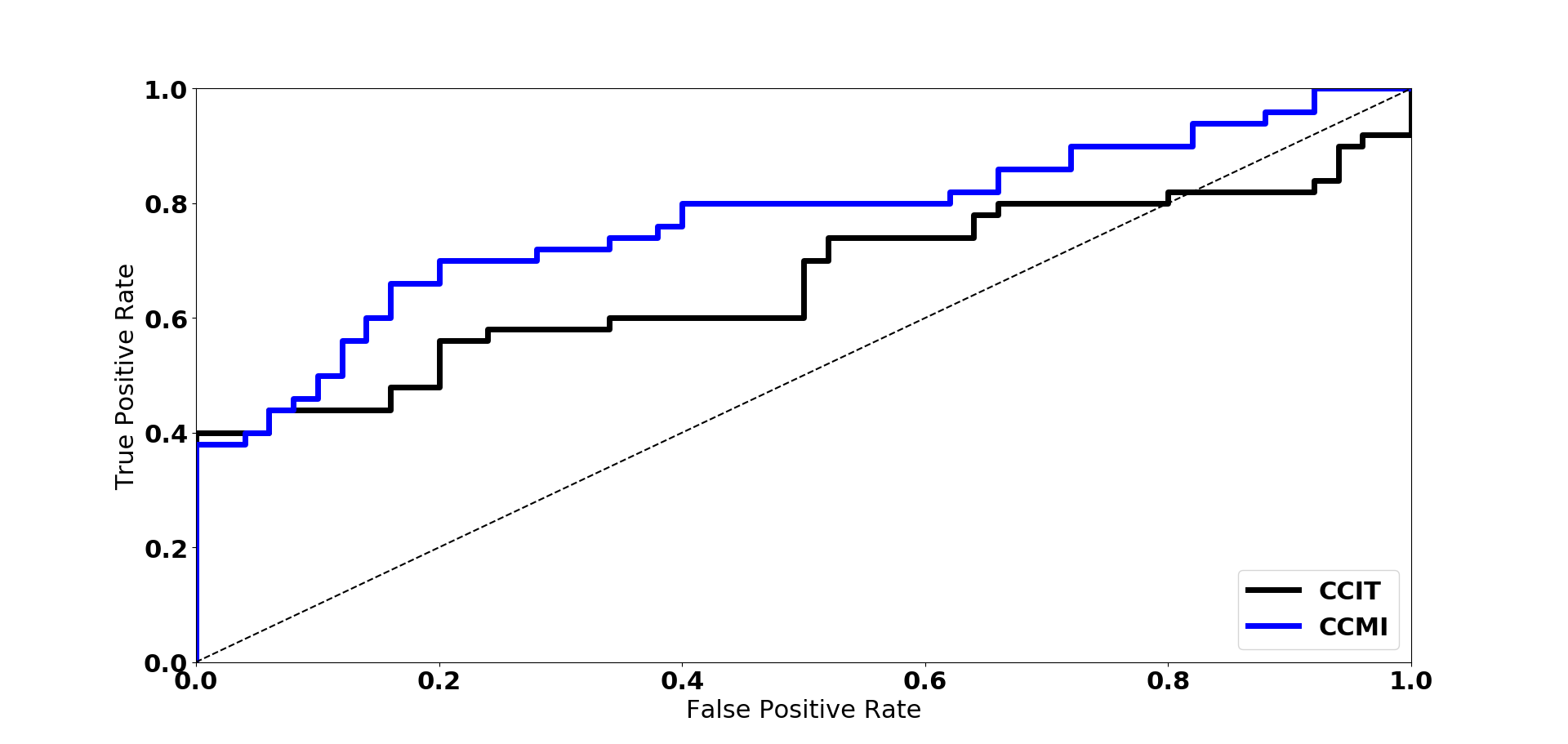} 
\caption{AuROC Curves : Flow-Cytometry Data-set. CCIT obtains a mean AuROC score of $0.6665$, while CCMI out-performs with mean of $0.7569$.}
\label{flowCyto-CIT}
\vspace{-0.2in}
\end{figure}
\vspace{-0.2in}
\section{Application to Conditional Independence Testing}
\vspace{-0.1in}
As a testimony to accurate CMI estimation, we apply CCMI to the problem of Conditional Independence Testing(CIT). Here, we are given samples from two distributions $p(x,y,z)$ and $q(x,y,z) = p(x,z)p(y|z)$. The hypothesis testing in CIT is to distinguish the null $\mathcal{H}_0 : X \perp Y |Z$ from the alternative $\mathcal{H}_1 : X \not\perp Y | Z$.

We seek to design a CIT tester using CMI estimation by using the fact that $I(X;Y|Z) = 0 \, \iff X \perp Y|Z$. A simple approach would be to reject the null if $I(X;Y|Z) > 0$ and accept it otherwise. The CMI estimates can serve as a proxy for the $P$-value. CIT testing based on CMI Estimation has been studied by \cite{runge2017conditional}, where the author uses KSG for CMI estimation and use $k$-NN based permutation to generate a $P$-value. The $P$-value is computed as the fraction of permuted data-sets where the CMI estimate is $\geq$ that of the original data-set. The same approach can be adopted for CCMI to obtain a $P$-value. But since we report the AuROC (Area under the Receiver Operating Characteristic curve), CMI estimates suffice.
\vspace{-0.2in}
\subsection{Post Non-linear Noise : Synthetic Data}
\vspace{-0.1in}
In this experiment, we generate data based on the post non-linear noise model similar to \cite{ccit}. As before, $d_x = d_y=1$ and $d_z$ can scale in dimension. The data is generated using the follow model. 
\begin{align*}
&Z \sim \mathcal{N}(\mathbbm{1}, I_{d_z}), \,\, X = cos(a_x Z + \eta_1) \\
&Y = 
\begin{cases}
cos(b_y Z + \eta_2) \qquad \qquad \,\,\,\textrm{if}  X \perp Y |Z \\
cos(c X + b_yZ + \eta_2) \qquad \textrm{if} \,\, X \not\perp Y |Z \\
\end{cases}
\end{align*}
The entries of random vectors(matrices if $d_x, d_y > 1$) $a_x$ and $b_y$ are drawn $\sim \mathcal{U}(0,1)$ and the vectors are normalized to have unit norm, i.e., $\|a\|_2 = 1, \|b\|_2 = 1$. $c \sim \mathcal{U}[0,2], \eta_i \sim \mathcal{N}(0, \sigma_e^2), \sigma_e = 0.5$. This is different from the implementation in \cite{ccit} where the constant is $c = 2$ in all data-sets. But by varying $c$, we obtain a tougher problem where the true CMI value can be quite low for a dependent data-set.

$a_x, b_y$ and $c$ are kept constant for generating points for a single data-set and are varied across data-sets. We vary $d_z \in \{1, 5, 20, 50, 70, 100 \}$ and simulate $100$ data-sets for each dimension. The number of samples is $n = 5000$ in each data-set. Our algorithm is compared with the state-of-the-art CIT tester in \cite{ccit}, known as CCIT. We used the implementation provided by the authors and ran CCIT with $B=50$ bootstraps \footnote{\url{https://github.com/rajatsen91/CCIT}}. For each data-set, an AuROC value is obtained. Figure \ref{postNonLin-CIT} shows the mean AuROC values from $5$ runs for both the testers as $d_z$ varies. While both algorithms perform accurately upto $d_z = 20$, the performance of CCIT starts to degrade beyond $20$ dimensions. Beyond $50$ dimensions, it performs close to random guessing. CCMI retains its superior performance even at $d_z = 100$, obtaining a mean AuROC value of $0.91$. 

Since AuROC metric finds best performance by varying thresholds, it is not clear what precision and recall is obtained from CCMI when we threshold the CCMI estimate at $0$ (and reject or accept the null based on it). So, for $d_z = 100$ we plotted the histogram of CMI estimates separately for CI and non-CI data-sets. Figure \ref{cmi-stat} shows that there a clear demarcation of CMI estimates between the two data-set categories and choosing the threshold as $0.0$ gave the precision as $0.84$ and recall as $0.86$.
\vspace{-0.2in}
\subsection{Flow-Cytometry : Real Data}
\vspace{-0.1in}
To extend our estimator beyond simulated settings, we use CMI estimation to test for conditional independence in the protein network data used in \cite{ccit}. The consensus graph in \cite{sachs} is used as the ground truth. We obtained $50$ CI and $50$ non-CI relations from the Bayesian network. The basic philosophy used is that a protein $X$ is independent of all other proteins $Y$ in the network given its parents, children and parents of children. Moreover, in the case of non-CI, we notice that a direct edge between $X$ and $Y$ would never render them conditionally independent. So the conditioning set $Z$ can be chosen at random from other proteins. These two settings are used to obtain the CI and non-CI data-sets. The number of samples in each data-set is only $853$ and the dimension of $Z$ varies from $5$ to $7$. 

For Flow-Cytometry data, since the number of samples is too small, we train the Classifier for fewer epochs to prevent over-fitting, keeping every other hyper-parameter the same. CCMI is compared with CCIT on the real data and the mean AuROC curves from $5$ runs is plotted in Figure \ref{flowCyto-CIT}. The superior performance of CCMI over CCIT is retained in sparse data regime.

\vspace{-0.2in}
\section{Conclusion and Future Directions}
\vspace{-0.15in}
In this work we explored various CMI estimators by drawing from recent advances in generative models and classifies. We proposed a new divergence estimator, based on Classifier-based two-sample estimation, and built several  conditional mutual information estimators using this primitive. We demonstrated their efficacy in a variety of practical settings.  Future work will aim to approximate the null distribution for CCMI, so that we can compute $P$-values for the conditional independence testing problem efficiently.

\section{Acknowledgments}

This work was supported by NSF awards 1651236 and
1703403 and NIH grant 5R01HG008164.

\bibliography{Ref}

\begin{thebibliography}{50}
\providecommand{\natexlab}[1]{#1}
\providecommand{\url}[1]{\texttt{#1}}
\expandafter\ifx\csname urlstyle\endcsname\relax
  \providecommand{\doi}[1]{doi: #1}\else
  \providecommand{\doi}{doi: \begingroup \urlstyle{rm}\Url}\fi

\bibitem[Beirlant et~al.(1997)Beirlant, Dudewicz, Gy{\"o}rfi, and Van~der
  Meulen]{beirlant1997nonparametric}
Jan Beirlant, Edward~J Dudewicz, L{\'a}szl{\'o} Gy{\"o}rfi, and Edward~C
  Van~der Meulen.
\newblock Nonparametric entropy estimation: An overview.
\newblock \emph{International Journal of Mathematical and Statistical
  Sciences}, 6\penalty0 (1):\penalty0 17--39, 1997.

\bibitem[Belghazi et~al.(2018)Belghazi, Baratin, Rajeshwar, Ozair, Bengio,
  Courville, and Hjelm]{belghazi}
Mohamed~Ishmael Belghazi, Aristide Baratin, Sai Rajeshwar, Sherjil Ozair,
  Yoshua Bengio, Aaron Courville, and Devon Hjelm.
\newblock Mutual information neural estimation.
\newblock In \emph{Proceedings of the 35th International Conference on Machine
  Learning}, 2018.

\bibitem[Doran et~al.()Doran, Muandet, Zhang, and Sch{\"o}lkopf]{kcit}
G~Doran, K~Muandet, K~Zhang, and B~Sch{\"o}lkopf.
\newblock A permutation-based kernel conditional independence test.
\newblock In \emph{30th Conference on Uncertainty in Artificial Intelligence
  (UAI 2014)}.

\bibitem[Entner and Hoyer(2012)]{entner}
Doris Entner and Patrik~O Hoyer.
\newblock Estimating a causal order among groups of variables in linear models.
\newblock In \emph{International Conference on Artificial Neural Networks},
  pages 84--91. Springer, 2012.

\bibitem[Fleuret(2004)]{fleuret}
Fran{\c{c}}ois Fleuret.
\newblock Fast binary feature selection with conditional mutual information.
\newblock \emph{Journal of Machine learning research}, 5\penalty0
  (Nov):\penalty0 1531--1555, 2004.

\bibitem[Frenzel and Pompe(2007)]{frenzel2007partial}
Stefan Frenzel and Bernd Pompe.
\newblock Partial mutual information for coupling analysis of multivariate time
  series.
\newblock \emph{Physical review letters}, 99\penalty0 (20):\penalty0 204101,
  2007.

\bibitem[Gao et~al.(2015)Gao, Ver~Steeg, and Galstyan]{gao2015efficient}
Shuyang Gao, Greg Ver~Steeg, and Aram Galstyan.
\newblock Efficient estimation of mutual information for strongly dependent
  variables.
\newblock In \emph{Artificial Intelligence and Statistics}, pages 277--286,
  2015.

\bibitem[Gao et~al.(2016)Gao, Oh, and Viswanath]{gao2016breaking}
Weihao Gao, Sewoong Oh, and Pramod Viswanath.
\newblock Breaking the bandwidth barrier: Geometrical adaptive entropy
  estimation.
\newblock In \emph{Advances in Neural Information Processing Systems}, pages
  2460--2468, 2016.

\bibitem[Gao et~al.(2017)Gao, Kannan, Oh, and Viswanath]{gao2017mixture}
Weihao Gao, Sreeram Kannan, Sewoong Oh, and Pramod Viswanath.
\newblock Estimating mutual information for discrete-continuous mixtures.
\newblock In \emph{Advances in Neural Information Processing Systems}, pages
  5988--5999, 2017.

\bibitem[Gao et~al.(2018)Gao, Oh, and Viswanath]{gao2018demystifying}
Weihao Gao, Sewoong Oh, and Pramod Viswanath.
\newblock Demystifying fixed $ k $-nearest neighbor information estimators.
\newblock \emph{IEEE Transactions on Information Theory}, 64\penalty0
  (8):\penalty0 5629--5661, 2018.

\bibitem[Giorgi et~al.(2014)Giorgi, Lopez, Woo, Bisikirska, Califano, and
  Bansal]{giorgi2014protein}
Federico~M Giorgi, Gonzalo Lopez, Jung~H Woo, Brygida Bisikirska, Andrea
  Califano, and Mukesh Bansal.
\newblock Inferring protein modulation from gene expression data using
  conditional mutual information.
\newblock \emph{PloS one}, 9\penalty0 (10):\penalty0 e109569, 2014.

\bibitem[Goodfellow et~al.(2014)Goodfellow, Pouget-Abadie, Mirza, Xu,
  Warde-Farley, Ozair, Courville, and Bengio]{goodfellow}
Ian Goodfellow, Jean Pouget-Abadie, Mehdi Mirza, Bing Xu, David Warde-Farley,
  Sherjil Ozair, Aaron Courville, and Yoshua Bengio.
\newblock Generative adversarial nets.
\newblock In \emph{Advances in neural information processing systems}, pages
  2672--2680, 2014.

\bibitem[Guo et~al.(2017)Guo, Pleiss, Sun, and Weinberger]{guo2017calibration}
Chuan Guo, Geoff Pleiss, Yu~Sun, and Kilian~Q Weinberger.
\newblock On calibration of modern neural networks.
\newblock In \emph{Proceedings of the 34th International Conference on Machine
  Learning-Volume 70}, pages 1321--1330. JMLR. org, 2017.

\bibitem[Hlinka et~al.(2013)Hlinka, Hartman, Vejmelka, Runge, Marwan, Kurths,
  and Palu{\v{s}}]{hlinka}
Jaroslav Hlinka, David Hartman, Martin Vejmelka, Jakob Runge, Norbert Marwan,
  J{\"u}rgen Kurths, and Milan Palu{\v{s}}.
\newblock Reliability of inference of directed climate networks using
  conditional mutual information.
\newblock \emph{Entropy}, 15\penalty0 (6):\penalty0 2023--2045, 2013.

\bibitem[Hornik et~al.(1989)Hornik, Stinchcombe, and White]{hornik}
Kurt Hornik, Maxwell Stinchcombe, and Halbert White.
\newblock Multilayer feedforward networks are universal approximators.
\newblock \emph{Neural networks}, 2\penalty0 (5):\penalty0 359--366, 1989.

\bibitem[Jiao et~al.(2018)Jiao, Gao, and Han]{jiao2017nearest}
Jiantao Jiao, Weihao Gao, and Yanjun Han.
\newblock The nearest neighbor information estimator is adaptively near minimax
  rate-optimal.
\newblock In \emph{Advances in neural information processing systems}, 2018.

\bibitem[Kandasamy et~al.(2015)Kandasamy, Krishnamurthy, Poczos, Wasserman,
  et~al.]{kandasamy2015nonparametric}
Kirthevasan Kandasamy, Akshay Krishnamurthy, Barnabas Poczos, Larry Wasserman,
  et~al.
\newblock Nonparametric von mises estimators for entropies, divergences and
  mutual informations.
\newblock In \emph{Advances in Neural Information Processing Systems}, pages
  397--405, 2015.

\bibitem[Kingma and Welling(2013)]{kingma}
Diederik~P Kingma and Max Welling.
\newblock Auto-encoding variational bayes.
\newblock \emph{arXiv preprint arXiv:1312.6114}, 2013.

\bibitem[Kozachenko and Leonenko(1987)]{kozachenko1987sample}
LF~Kozachenko and Nikolai~N Leonenko.
\newblock Sample estimate of the entropy of a random vector.
\newblock \emph{Problemy Peredachi Informatsii}, 23\penalty0 (2):\penalty0
  9--16, 1987.

\bibitem[Kraskov et~al.(2004)Kraskov, St{\"o}gbauer, and Grassberger]{kraskov}
Alexander Kraskov, Harald St{\"o}gbauer, and Peter Grassberger.
\newblock Estimating mutual information.
\newblock \emph{Physical review E}, 69\penalty0 (6):\penalty0 066138, 2004.

\bibitem[Lakshminarayanan et~al.(2017)Lakshminarayanan, Pritzel, and
  Blundell]{lakshminarayanan}
Balaji Lakshminarayanan, Alexander Pritzel, and Charles Blundell.
\newblock Simple and scalable predictive uncertainty estimation using deep
  ensembles.
\newblock In \emph{Advances in Neural Information Processing Systems}, pages
  6402--6413, 2017.

\bibitem[Lee(2010)]{lee2010sample}
Intae Lee.
\newblock Sample-spacings-based density and entropy estimators for spherically
  invariant multidimensional data.
\newblock \emph{Neural Computation}, 22\penalty0 (8):\penalty0 2208--2227,
  2010.

\bibitem[Le{\'s}niewicz(2014)]{lesniewicz2014expected}
Marek Le{\'s}niewicz.
\newblock Expected entropy as a measure and criterion of randomness of binary
  sequences.
\newblock \emph{Przegl{\k{a}}d Elektrotechniczny}, 90\penalty0 (1):\penalty0
  42--46, 2014.

\bibitem[Li et~al.(2011)Li, Ouyang, Li, and Li]{Licausal}
Zhaohui Li, Gaoxiang Ouyang, Duan Li, and Xiaoli Li.
\newblock Characterization of the causality between spike trains with
  permutation conditional mutual information.
\newblock \emph{Physical Review E}, 84\penalty0 (2):\penalty0 021929, 2011.

\bibitem[Liang and Wang(2008)]{liang2008gene}
Kuo-Ching Liang and Xiaodong Wang.
\newblock Gene regulatory network reconstruction using conditional mutual
  information.
\newblock \emph{EURASIP Journal on Bioinformatics and Systems Biology},
  2008\penalty0 (1):\penalty0 253894, 2008.

\bibitem[Loeckx et~al.(2010)Loeckx, Slagmolen, Maes, Vandermeulen, and
  Suetens]{loeckx}
Dirk Loeckx, Pieter Slagmolen, Frederik Maes, Dirk Vandermeulen, and Paul
  Suetens.
\newblock Nonrigid image registration using conditional mutual information.
\newblock \emph{IEEE transactions on medical imaging}, 29\penalty0
  (1):\penalty0 19--29, 2010.

\bibitem[Lopez-Paz and Oquab(2016)]{lopezC2ST}
David Lopez-Paz and Maxime Oquab.
\newblock Revisiting classifier two-sample tests.
\newblock \emph{arXiv preprint arXiv:1610.06545}, 2016.

\bibitem[Miller(2003)]{miller2003new}
Erik~G Miller.
\newblock A new class of entropy estimators for multi-dimensional densities.
\newblock In \emph{Acoustics, Speech, and Signal Processing, 2003.
  Proceedings.(ICASSP'03). 2003 IEEE International Conference on}, volume~3,
  pages III--297. IEEE, 2003.

\bibitem[Mirza and Osindero(2014)]{mirza}
Mehdi Mirza and Simon Osindero.
\newblock Conditional generative adversarial nets.
\newblock \emph{arXiv preprint arXiv:1411.1784}, 2014.

\bibitem[Mohri et~al.(2018)Mohri, Rostamizadeh, and
  Talwalkar]{mohri2018foundations}
Mehryar Mohri, Afshin Rostamizadeh, and Ameet Talwalkar.
\newblock \emph{Foundations of machine learning}.
\newblock 2018.

\bibitem[Nemenman et~al.(2002)Nemenman, Shafee, and
  Bialek]{nemenman2002entropy}
Ilya Nemenman, Fariel Shafee, and William Bialek.
\newblock Entropy and inference, revisited.
\newblock In \emph{Advances in neural information processing systems}, pages
  471--478, 2002.

\bibitem[Nguyen et~al.(2008)Nguyen, Wainwright, and Jordan]{nguyen2008}
XuanLong Nguyen, Martin~J Wainwright, and Michael~I Jordan.
\newblock Estimating divergence functionals and the likelihood ratio by
  penalized convex risk minimization.
\newblock In \emph{Advances in neural information processing systems}, pages
  1089--1096, 2008.

\bibitem[Niculescu-Mizil and Caruana(2005)]{niculescu}
Alexandru Niculescu-Mizil and Rich Caruana.
\newblock Obtaining calibrated probabilities from boosting.
\newblock In \emph{UAI}, 2005.

\bibitem[P{\'a}l et~al.(2010)P{\'a}l, P{\'o}czos, and
  Szepesv{\'a}ri]{pal2010estimation}
D{\'a}vid P{\'a}l, Barnab{\'a}s P{\'o}czos, and Csaba Szepesv{\'a}ri.
\newblock Estimation of r{\'e}nyi entropy and mutual information based on
  generalized nearest-neighbor graphs.
\newblock In \emph{Advances in Neural Information Processing Systems}, pages
  1849--1857, 2010.

\bibitem[Parviainen and Kaski(2016)]{parviainen}
Pekka Parviainen and Samuel Kaski.
\newblock Bayesian networks for variable groups.
\newblock In \emph{Conference on Probabilistic Graphical Models}, pages
  380--391, 2016.

\bibitem[Poole et~al.()Poole, Ozair, van~den Oord, Alemi, and Tucker]{poole}
Ben Poole, Sherjil Ozair, A{\"a}ron van~den Oord, Alexander~A Alemi, and George
  Tucker.
\newblock On variational lower bounds of mutual information.

\bibitem[Rahimzamani et~al.(2018)Rahimzamani, Asnani, Viswanath, and
  Kannan]{gdmnips2018}
Arman Rahimzamani, Himanshu Asnani, Pramod Viswanath, and Sreeram Kannan.
\newblock Estimators for multivariate information measures in general
  probability spaces.
\newblock In \emph{Advances in Neural Information Processing Systems 31}.
  Curran Associates, Inc., 2018.

\bibitem[Runge(2018)]{runge2017conditional}
Jakob Runge.
\newblock Conditional independence testing based on a nearest-neighbor
  estimator of conditional mutual information.
\newblock In \emph{Proceedings of the Twenty-First International Conference on
  Artificial Intelligence and Statistics}, 2018.

\bibitem[Sachs et~al.(2005)Sachs, Perez, Pe'er, Lauffenburger, and
  Nolan]{sachs}
Karen Sachs, Omar Perez, Dana Pe'er, Douglas~A Lauffenburger, and Garry~P
  Nolan.
\newblock Causal protein-signaling networks derived from multiparameter
  single-cell data.
\newblock \emph{Science}, 308\penalty0 (5721):\penalty0 523--529, 2005.

\bibitem[Segal et~al.(2005)Segal, Pe’er, Regev, Koller, and Friedman]{segal}
Eran Segal, Dana Pe’er, Aviv Regev, Daphne Koller, and Nir Friedman.
\newblock Learning module networks.
\newblock \emph{Journal of Machine Learning Research}, 6\penalty0
  (Apr):\penalty0 557--588, 2005.

\bibitem[Sen et~al.(2017)Sen, Suresh, Shanmugam, Dimakis, and Shakkottai]{ccit}
Rajat Sen, Ananda~Theertha Suresh, Karthikeyan Shanmugam, Alexandros~G Dimakis,
  and Sanjay Shakkottai.
\newblock Model-powered conditional independence test.
\newblock In \emph{Advances in Neural Information Processing Systems}, pages
  2951--2961, 2017.

\bibitem[Sen et~al.(2018)Sen, Shanmugam, Asnani, Rahimzamani, and
  Kannan]{sen2018mimic}
Rajat Sen, Karthikeyan Shanmugam, Himanshu Asnani, Arman Rahimzamani, and
  Sreeram Kannan.
\newblock Mimic and classify: A meta-algorithm for conditional independence
  testing.
\newblock \emph{arXiv preprint arXiv:1806.09708}, 2018.

\bibitem[Singh et~al.(2003)Singh, Misra, Hnizdo, Fedorowicz, and
  Demchuk]{singh2003nearest}
Harshinder Singh, Neeraj Misra, Vladimir Hnizdo, Adam Fedorowicz, and Eugene
  Demchuk.
\newblock Nearest neighbor estimates of entropy.
\newblock \emph{American journal of mathematical and management sciences},
  23\penalty0 (3-4):\penalty0 301--321, 2003.

\bibitem[Singh and P{\'o}czos(2014)]{singh2014exponential}
Shashank Singh and Barnab{\'a}s P{\'o}czos.
\newblock Exponential concentration of a density functional estimator.
\newblock In \emph{Advances in Neural Information Processing Systems}, pages
  3032--3040, 2014.

\bibitem[Singh and P{\'o}czos(2016)]{singh2016finite}
Shashank Singh and Barnab{\'a}s P{\'o}czos.
\newblock Finite-sample analysis of fixed-k nearest neighbor density functional
  estimators.
\newblock In \emph{Advances in Neural Information Processing Systems}, pages
  1217--1225, 2016.

\bibitem[Sohn et~al.(2015)Sohn, Lee, and Yan]{cvae}
Kihyuk Sohn, Honglak Lee, and Xinchen Yan.
\newblock Learning structured output representation using deep conditional
  generative models.
\newblock In \emph{Advances in Neural Information Processing Systems 28}. 2015.

\bibitem[Sricharan et~al.(2012)Sricharan, Raich, and
  Hero]{sricharan2012estimation}
Kumar Sricharan, Raviv Raich, and Alfred~O Hero.
\newblock Estimation of nonlinear functionals of densities with confidence.
\newblock \emph{IEEE Transactions on Information Theory}, 58\penalty0
  (7):\penalty0 4135--4159, 2012.

\bibitem[Sricharan et~al.(2013)Sricharan, Wei, and Hero]{sricharan2013ensemble}
Kumar Sricharan, Dennis Wei, and Alfred~O Hero.
\newblock Ensemble estimators for multivariate entropy estimation.
\newblock \emph{IEEE transactions on information theory}, 59\penalty0
  (7):\penalty0 4374--4388, 2013.

\bibitem[Suzuki et~al.(2008)Suzuki, Sugiyama, Sese, and Kanamori]{suzuki}
Taiji Suzuki, Masashi Sugiyama, Jun Sese, and Takafumi Kanamori.
\newblock Approximating mutual information by maximum likelihood density ratio
  estimation.
\newblock In \emph{New challenges for feature selection in data mining and
  knowledge discovery}, pages 5--20, 2008.

\bibitem[Vejmelka and Palu{\v{s}}(2008)]{vejmelka2008inferring}
Martin Vejmelka and Milan Palu{\v{s}}.
\newblock Inferring the directionality of coupling with conditional mutual
  information.
\newblock \emph{Physical Review E}, 77\penalty0 (2):\penalty0 026214, 2008.

\end{thebibliography}

\newpage 
\section{Supplementary}

\subsection{Hyper-parameters}
We provide the experimental settings and hyper-parameters for ease of reproducibility of the results. 

\begin{table}[h!]
\caption{Classifier : Hyper-parameters}
\label{classifier-param}
\begin{center}
\begin{tabular}{ll}
\multicolumn{1}{c}{\bf Hyper-parameter}  &\multicolumn{1}{c}{\bf Value} \\
\hline \\
Hidden Units & $ 64 $\\
\# Hidden Layers & $2$ (Inp-64-64-Out)\\
Activation & ReLU\\
Batch-Size & $64$ \\
Learning Rate & $0.001$  \\
Optimizer & Adam \\
& ($\beta_1 = 0.90, \beta_2 = 0.999$) \\
\# Epoch & $20$ \\
Regularizer & L2 ($0.001$) \\
\end{tabular}
\end{center}
\end{table}

\begin{table}[h!]
\caption{CGAN : Hyper-parameters}
\label{cgan-param}
\begin{center}
\begin{tabular}{ll}
\multicolumn{1}{c}{\bf Hyper-parameter}  &\multicolumn{1}{c}{\bf Value} \\
\hline \\
Hidden Units & $ 256 $\\
\# Hidden Layers & $2$ (Inp-256-256-Out)\\
Activation & Leaky ReLU($0.2$)\\
Batch-Size & $128$ \\
Learning Rate & $1e-4$  \\
Optimizer & Adam \\
& ($\beta_1 = 0.5, \beta_2 = 0.9$) \\
\# Epoch & $100$ \\
Noise dimension & $ 20 $\\
Noise distribution & $\mathcal{U}(-1.0, 1.0)^{d_s}$
\end{tabular}
\end{center}
\end{table}

\begin{table}[h!]
\caption{CVAE : Hyper-parameters}
\label{cvae-param}
\begin{center}
\begin{tabular}{ll}
\multicolumn{1}{c}{\bf Hyper-parameter}  &\multicolumn{1}{c}{\bf Value} \\
\hline \\
Hidden Units & $ 256 $\\
\# Hidden Layers & $2$ (Inp-256-256-Out)\\
Activation & Leaky ReLU($0.2$)\\
Batch-Size & $128$ \\
Learning Rate & $1e-4$  \\
Optimizer & Adam \\
& ($\beta_1 = 0.5, \beta_2 = 0.9$) \\
\# Epoch & $20$ \\
Dropout & $0.9$ \\
Latent dimension & $ 20 $\\
\end{tabular}
\end{center}
\end{table}

\begin{table}[h!]
\caption{f-MINE : Hyper-parameters}
\label{f-mine-param}
\begin{center}
\begin{tabular}{ll}
\multicolumn{1}{c}{\bf Hyper-parameter}  &\multicolumn{1}{c}{\bf Value} \\
\hline \\
Hidden Units & $ 64 $\\
\# Hidden Layers & $1$ (Inp-64-Out)\\
Activation & ReLU\\
Batch-Size & $128$ ($512$ for DV-MINE) \\
Learning Rate & $1e-4$  \\
Optimizer & Adam \\
& ($\beta_1 = 0.5, \beta_2 = 0.999$) \\
\# Epoch & $200$ \\
\end{tabular}
\end{center}
\end{table}

\begin{table}[h!]
\caption{AuROC : Flow-Cytometry Data (Mean $\pm$ Std. of 5 runs.}
\label{auroc-flow}
\begin{center}
\begin{tabular}{ll}
\multicolumn{1}{c}{\bf Tester}  &\multicolumn{1}{c}{\bf AuROC} \\
\hline \\
CCIT &$ 0.6665 \pm 0.006 $ \\
CCMI &$ \textbf{0.7569} \pm 0.047 $ \\
\end{tabular}
\end{center}
\end{table}

\subsection{Calibration Curve}

\begin{figure}[h!]
\includegraphics[width=0.47\textwidth]{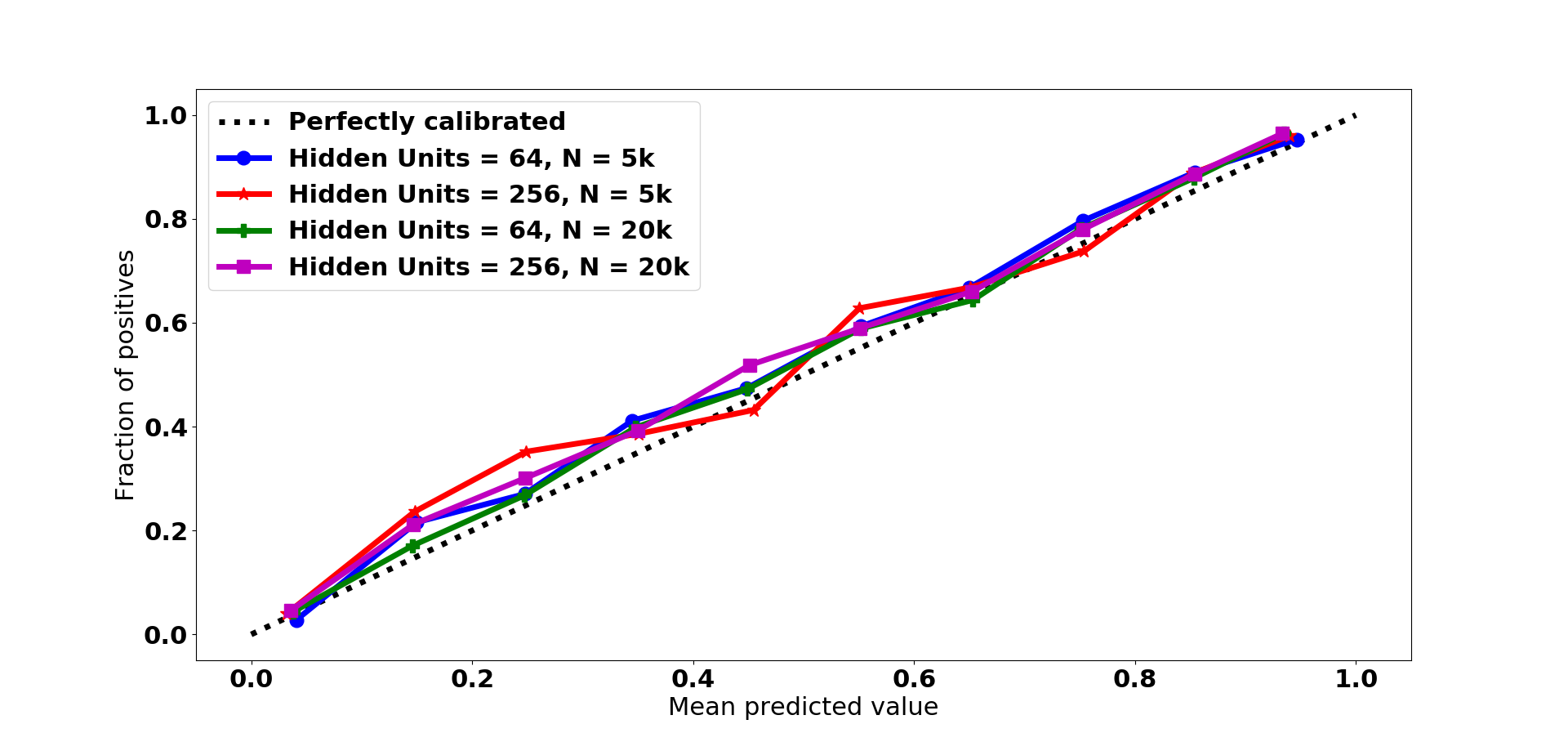} 
\caption{Calibrated Classifiers : We find that our classifiers trained with $L2$-regularization and two hidden layers are well-calibrated. The calibration is obtained for MI Estimation of Correlated Gaussians with $d_x = 10, \rho = 0.5$}
\label{calibration-nn}
\end{figure}

While \cite{niculescu} showed that neural networks for binary classification produce well-calibrated outputs. the authors in \cite{guo2017calibration} found miscalibration in deep networks with batch-normalization and no L2 regularization. In our experiments, the classifier is shallow, consisting of only $2$ layers with relatively small number of hidden units. There is no batch-normalization or dropout used. Instead, we use $L2$-regularization which was shown in \cite{guo2017calibration} to be favorable for calibration. Figure \ref{calibration-nn} shows that our classifiers are well-calibrated. 

\subsection{Choosing Optimal Hyper-parameter}

\begin{figure*}[t]
\centering
\begin{subfigure}[b]{0.47\textwidth}
\includegraphics[width=\textwidth]{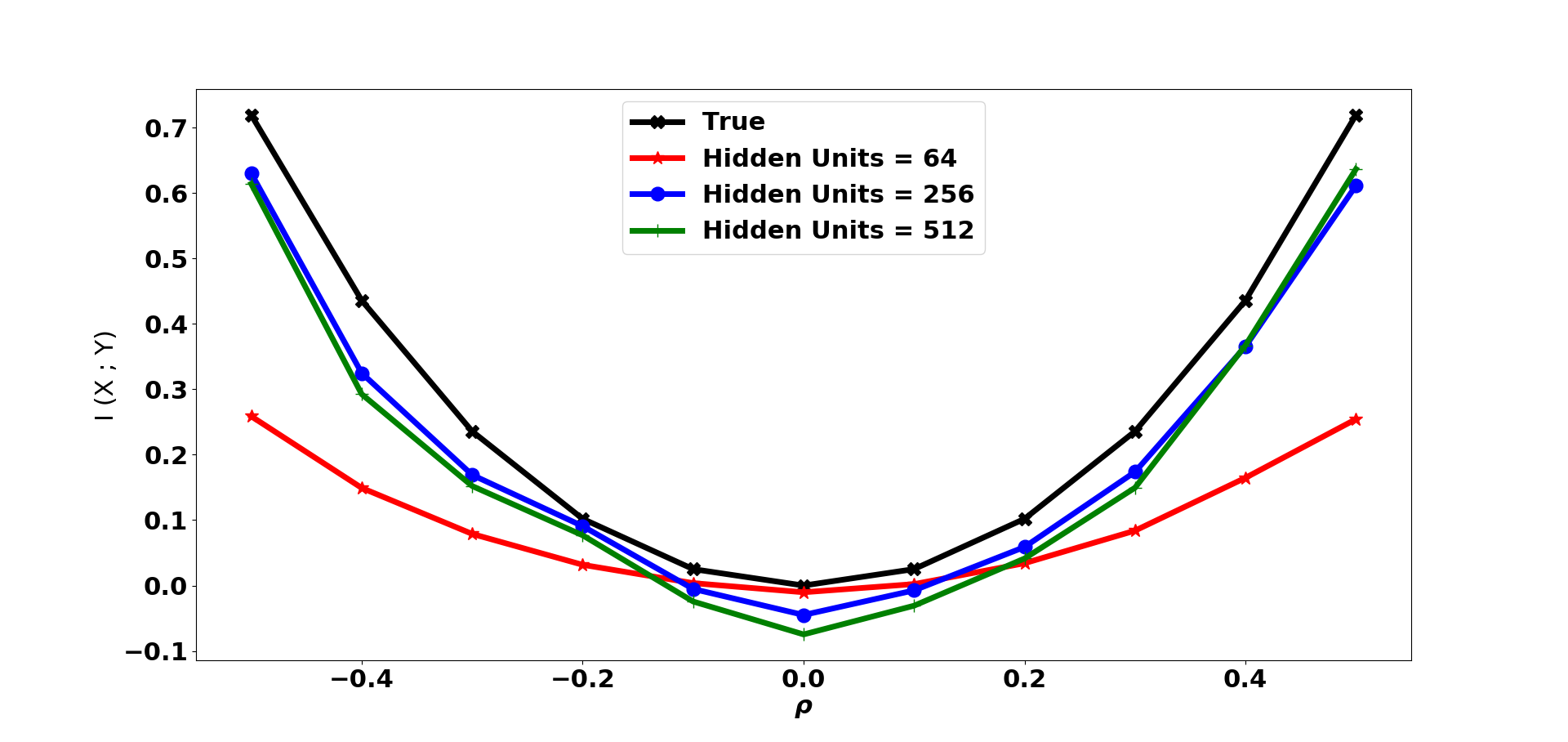} 
\caption{$d_x = d_y = 5, N = 500, $}
\end{subfigure}
\begin{subfigure}[b]{0.47\textwidth}
\includegraphics[width=\textwidth]{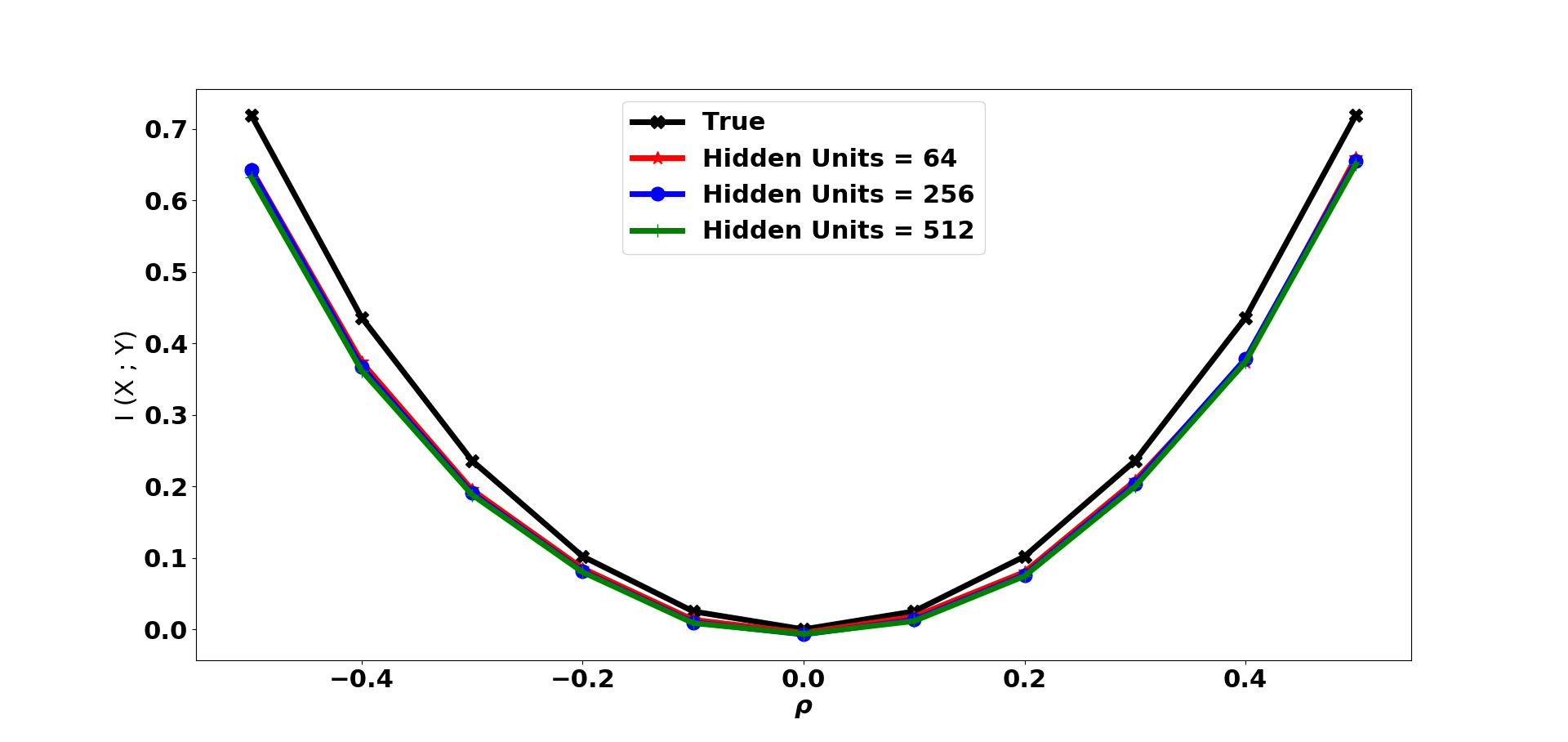} 
\caption{$d_x = d_y = 5, N = 5000$}
\end{subfigure}
\begin{subfigure}[b]{0.47\textwidth}
\includegraphics[width=\textwidth]{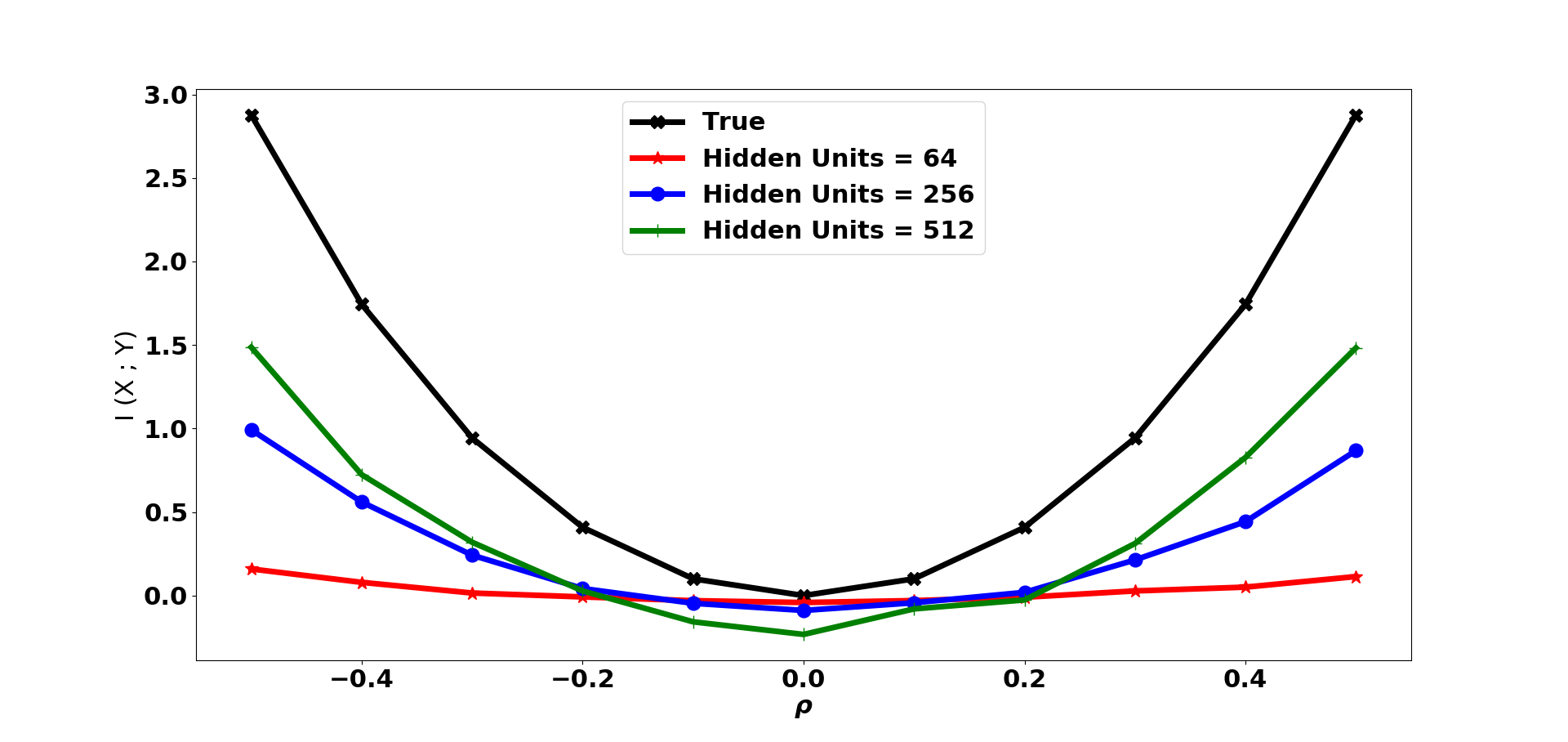} 
\caption{$d_x = d_y = 20, N = 500$}
\end{subfigure}
\begin{subfigure}[b]{0.47\textwidth}
\includegraphics[width=\textwidth]{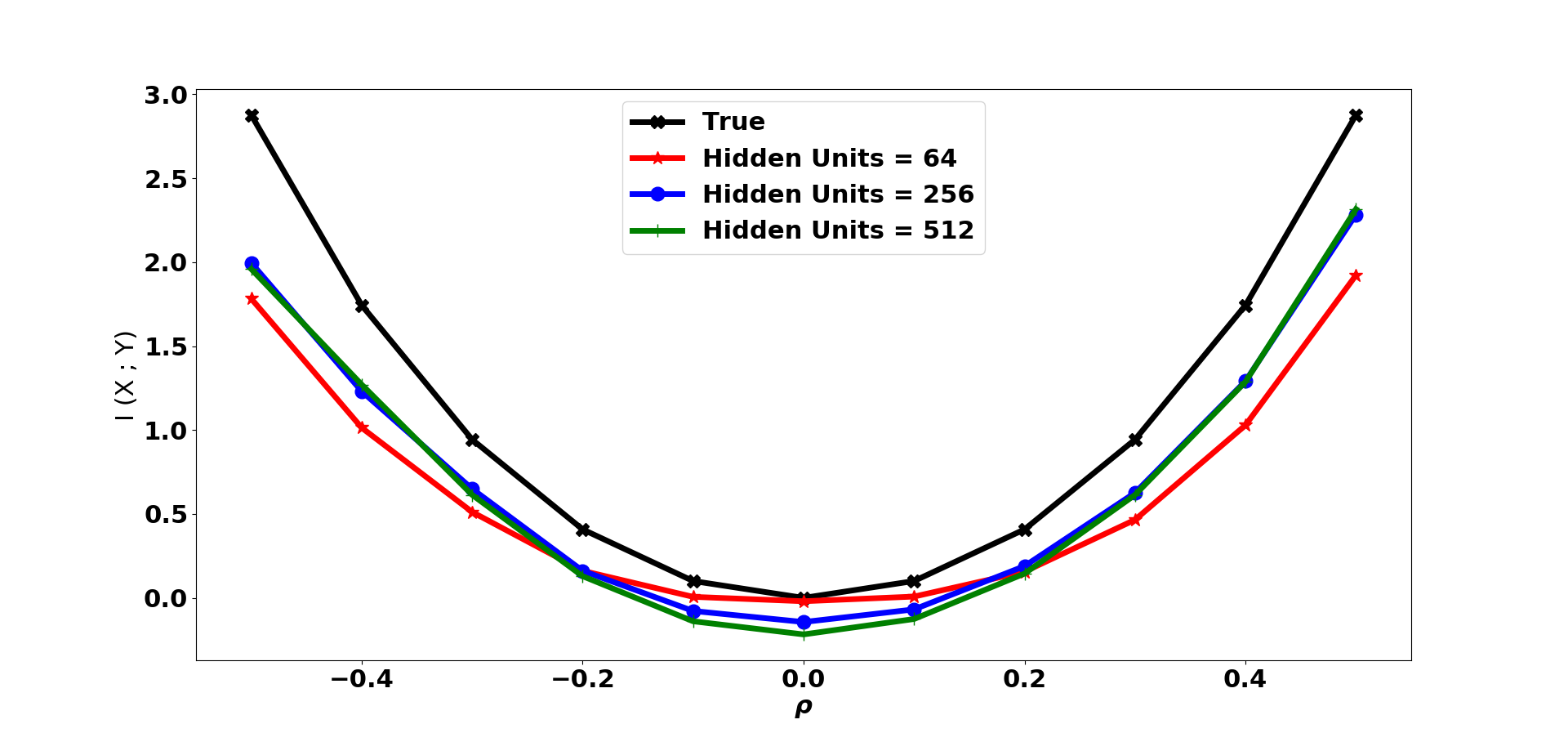} 
\caption{$d_x = d_y = 20, N = 5000$}
\end{subfigure}
\caption{The Donsker-Varadhan Representation provides a lower bound of the true MI. For each hyper-paramter choice, the estimates lie below $I^*(X;Y)$. An optimal estimator would return the maximum estimate from multiple hyper-parameter choices for a given data-set. Estimates are plotted for Correlation Gaussians introduced in Figure \ref{toy}.}
\label{hyper-param-gauss}
\end{figure*} 

The Donsker-Varadhan representation \ref{DV-bound} is a lower bound on the true MI estimate (which is the supremum over all functions). So, for any classifier parameter, the plug-in estimate value computed on the test samples will be less than or equal to the true value $I(X;Y)$ with high probability (Theorem \ref{theorem-true-lower-bound}). We illustrate this using estimation of MI for Correlated Gaussians in Figure \ref{hyper-param-gauss}. The estimated value lies below the true values of MI. Thus, the optimal hyper-parameter is the one that returns the maximum value of MI estimate on the test set. 

Once we have this block that returns the maximum MI estimate after searching over hyper-parameters, CMI estimate in CCMI is the difference of two MI estimates, calling this block twice.

We also plot the AuROC curves for the two choices of number of hidden units in flow-Cytometry data (Figure \ref{hyper-param-flow}) and post Non-linear noise synthetic data (Figure \ref{hyper-param-post}). When the number of samples is high, the estimates are pretty robust to hyper-parameter choice (Figure \ref{hyper-param-gauss} (b), \ref{hyper-param-post}). But in sparse sample regime, proper choice of hyper-parameter can improve performance (Figure \ref{hyper-param-flow}).

\begin{figure}[h!]
\includegraphics[width=0.47\textwidth]{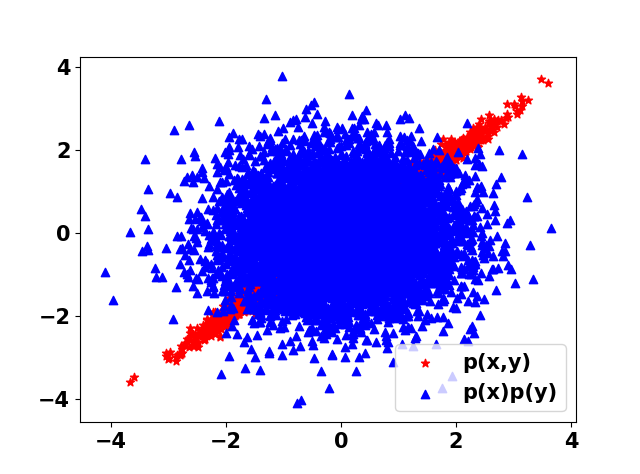} 
\caption{Logistic Regression Fails to Classify points from $p(x_1, x_2)$ (colored red) and those from $p(x_1)p(x_2)$ (colored blue).}
\label{log-reg}
\end{figure}

\begin{figure*}[t]
\centering
\begin{subfigure}[b]{0.47\textwidth}
\includegraphics[width=\textwidth]{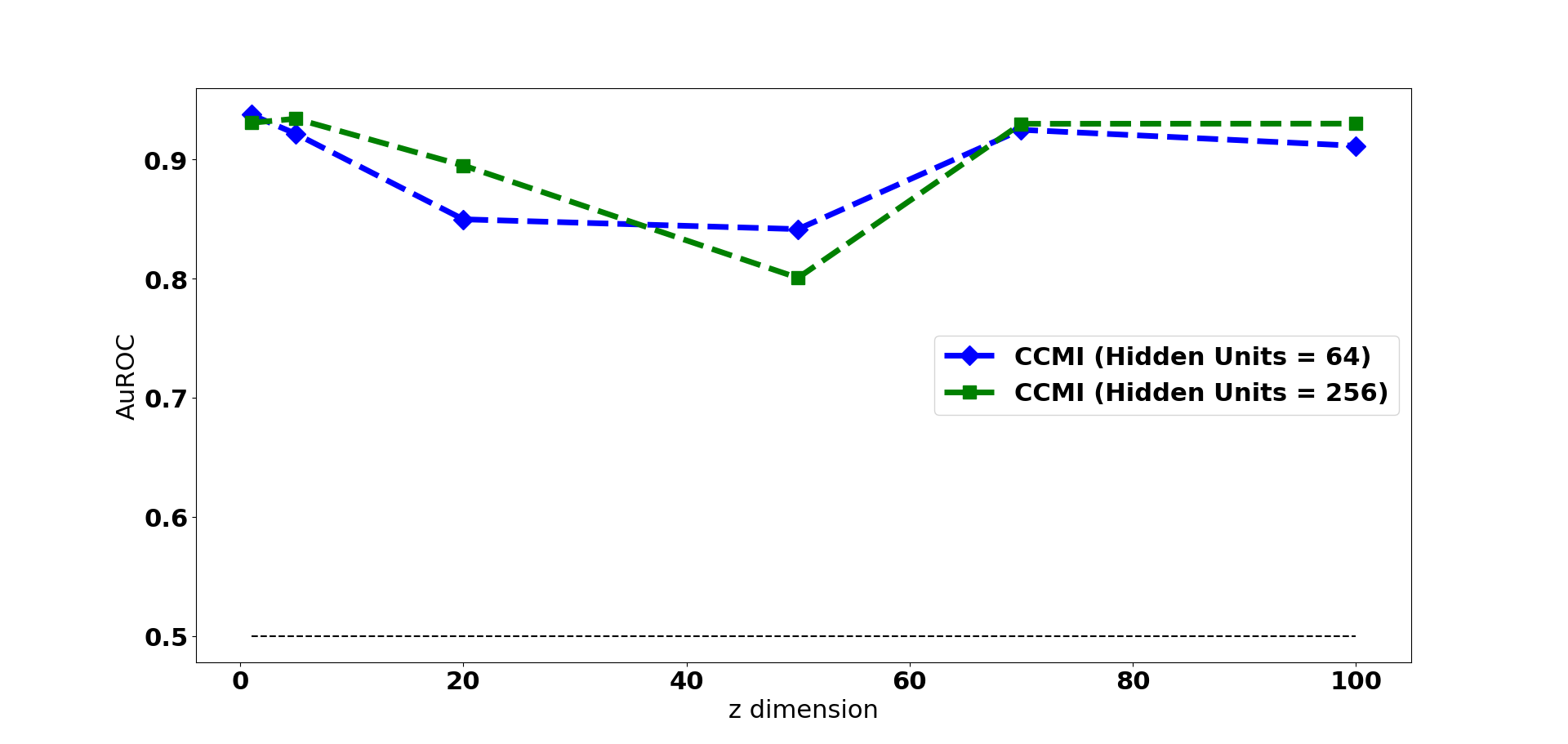} 
\caption{Post Non-linear noise data-sets}
\label{hyper-param-post}
\end{subfigure}
\begin{subfigure}[b]{0.47\textwidth}
\includegraphics[width=\textwidth]{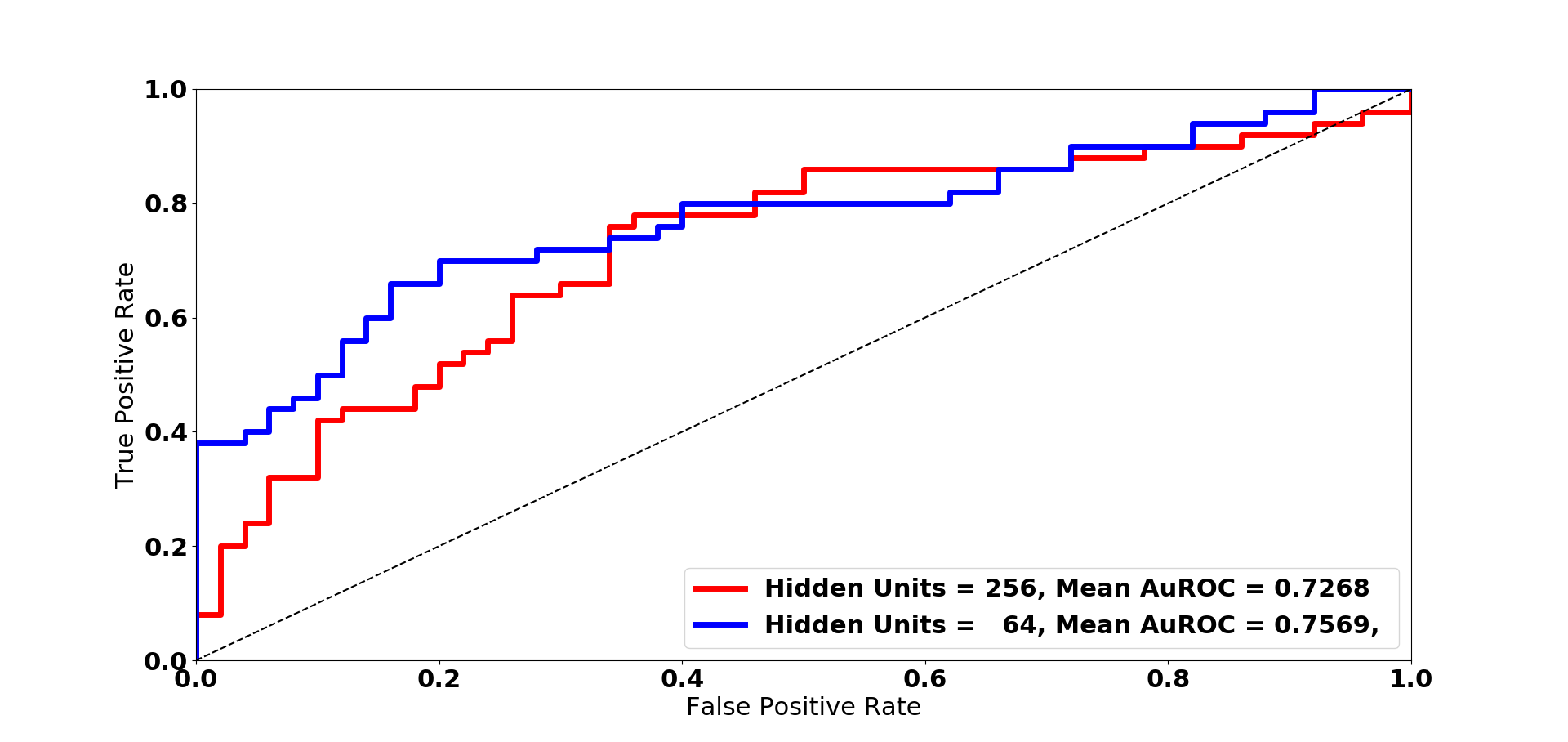} 
\caption{Flow-Cytometry data-sets}
\label{hyper-param-flow}
\end{subfigure}
\caption{Hyper-parameter Sensitivity : We observe the performance in conditional independence testing with number of hidden units as $64$ Vs $256$, keeping all other hyper-parameters the same in respective cases.}
\label{hyper-param-cit}
\end{figure*} 

\subsection{Additional Figures and Tables}

\begin{itemize}
    \item For Flow-Cytometry data-set, we used number of hidden units = $64$ for Classifier and trained for $10$ epochs. Table \ref{auroc-flow} shows the mean AuROC values for two CIT testers.
    \item Figure \ref{log-reg} shows the distribution of points from $p(x_1,x_2)$ and $p(x_1)p(x_2)$. Here the classifier would return $0.5$ as prediction for either class (leading to $\hat{D}_{KL}$ $= 0$), even though $X_1$ and $X_2$ are highly correlated ($\rho = 0.99$) and the mutual information is high.
\end{itemize}

\section{Theoretical Properties of CCMI}

In this Section, we explore some of the theoretical properties of CCMI. Let the samples $x_i \sim p(x)$ be labeled as $l = 1$ and $x_j \sim q(x)$ be labeled as $l = 0$. Let $Pr(l = 1) = Pr(l = 0) = 0.5$. The positive label probability for a given point $x$ is denoted as $\gamma(x) = Pr(l=1 | x)$. When the prediction is from a classifier with parameter $\theta$, then it is denoted as $\gamma_{\theta}(x)$. The argument $x$ of $\gamma$ is dropped when it is understood from the context.

 The following assumptions are used throughout this Section.
\begin{itemize}
\item Assumption (A1) : The underlying data distributions $p(\cdot)$ and $q(\cdot)$ admit densities in a compact subset $\mathcal{X} \subset \mathbb{R}^{d_x}$.
\item Assumption (A2) : $\exists \, \, \alpha, \beta > 0 $, such that $ \alpha \leq p(x), q(x) \leq \beta \, \, \forall \, x.$
\item Assumption (A3) : We clip predictions in algorithm such that $\gamma(x) \in  [\tau, 1-\tau ] \, \forall \, x$, with $0 < \tau \leq \alpha/(\alpha + \beta)$.
\item Assumption (A4) : The classifier class $\mathcal{C}_\theta$ is parameterized by $\theta$ in some compact domain $\Theta \subset \mathbb{R}^{h}$. $ \exists$ constant $K$,  such that $\| \theta \| \leq K$ and the output of the classifier is $L$-Lipschitz with respect to parameters $\theta$. 
\end{itemize}


\subsection*{Notation and Computation Procedure}
\begin{itemize}
\item In the case of mutual information estimation $I(U;V)$,  $x \in \mathbb{R}^{d_u + d_v} $ represents the concatenated data point $(u, v)$. To be precise, $p(x) = p(u,v)$ and $q(x) = p(u)p(v)$. \\
\item In the proofs below, we need to compute the Lipschitz constant for various functions. The general procedure for those computations are as follows. 
\[
|\phi(x) - \phi(y)| \leq L_\phi|x-y|
\]
We compute $L_\phi$ using $\sup_z |\phi'(z)|, z \in \mathrm{domain}(\phi)$. The functions encountered in the proofs are continuous, differentiable and have bounded domains.

\item The binary-cross entropy loss estimated from $n$ samples is 
\begin{multline}
\mathrm{BCE}_n (\gamma) = - \left( \frac{1}{n} \sum_i l_i \log \gamma(x_i) + \right.  \\
\biggl. (1-l_i) \log(1-\gamma(x_i)) \biggr)
\label{bce-test-n}
\end{multline}
When computed on the train samples (resp. test samples), it is denoted as $\mathrm{BCE}^{\mathrm{ERM}}_n (\gamma)$ (resp.  $\mathrm{BCE}_n (\gamma)$). 
The population mean over the joint distribution of data and labels is
\begin{multline}
\mathrm{BCE}(\gamma) = - \left( \mathbb{E}_{XL} L \log \gamma(X) + \right. \\
\left. (1-L) \log(1-\gamma(X)) \right)
\label{bce-population}
\end{multline}

\item The estimate of MI from $n$ test samples for classifier parameter $\hat{\theta}$ is given by
\[
I^{\gamma_{\hat{\theta}}}_n = \frac{1}{n}\sum\limits_{i=1}^n \log \frac{\gamma_{\hat{\theta}(x_i)}}{1-\gamma_{\hat{\theta}(x_i)}} - \log \left(  \frac{1}{n}\sum\limits_{j=1}^n \frac{\gamma_{\hat{\theta}}(x_j)}{1-\gamma_{\hat{\theta}}(x_j)} \right)
\]
The population estimate for classifier parameter $\hat{\theta}$ is given by
\[
I^{\gamma_{\hat{\theta}}} = \mathop{\mathbb{E}}\limits_{x \sim p} \log \frac{\gamma_{\hat{\theta}(x)}}{1-\gamma_{\hat{\theta}(x)}} - \log \left( \mathop{\mathbb{E}}\limits_{x \sim q}  \frac{\gamma_{\hat{\theta}}(x)}{1-\gamma_{\hat{\theta}}(x)} \right)
\]

\end{itemize}

\begin{theorem}[Theorem \ref{theorem-class-mi-consistent} restated]
Classifier-MI is consistent, i.e., given $\epsilon, \delta > 0, \exists \, n \in \mathbb{N} $, such that with probability at least $1-\delta$, we have 
\[
|I^{\gamma_{\hat{\theta}}}_n (U;V) - I(U;V) | \leq \epsilon
\]

\label{re-theorem-class-mi-consistent}
\end{theorem}

\subsection*{Intuition of Proof}

The classifier is trained to minimize the empirical risk on the train set and obtains the minimizer as $\hat{\theta}$. From generalization bound of classifier, this loss value ($\mathrm{BCE}(\gamma_{\hat{\theta}})$) on the test set is close to the loss obtained by the best optimizer in the classifier family ($\mathrm{BCE}(\gamma_{\tilde{\theta}})$), which itself is close to the loss from global optimizer $\gamma^*$ (viz. $\mathrm{BCE}(\gamma^*)$) by Universal Function Approximation Theorem of neural-networks. 

The $\mathrm{BCE}$ loss is strongly convex in $\gamma$. $\gamma$ links $\mathrm{BCE}$ to $I(\cdot \,;\, \cdot)$, i.e., $|\mathrm{BCE}_n(\gamma_{\hat{\theta}}) - \mathrm{BCE}(\gamma^*)| \leq \epsilon' \implies \| \gamma_{\hat{\theta}} - \gamma^*\|_1 \leq \eta \implies | \hat{I}_n (U;V) - I(U;V)| \leq \epsilon$.

\begin{lemma}[Likelihood-Ratio from Cross-Entropy Loss]
The point-wise minimizer of binary cross-entropy loss $\gamma^*(x)$ is related to the likelihood ratio as $\frac{\gamma^*(x)}{1-\gamma^*(x)} = \frac{p(x)}{q(x)}$, where $\gamma^*(x) = Pr(l=1|x)$ and $l$ is the label of point $x$.
\label{Lemma-BCE}
\end{lemma}

\begin{proof}

The binary cross entropy loss as a function of gamma is defined in (\ref{bce-population}).
Now, 
\begin{align*}
\mathbb{E}_{XL} L \log \gamma(X) &= \sum_{x,l} p(x,l) l \log \gamma(x) \\
&=  \sum_{x,l=1} p(x|l=1)p(l=1) \log \gamma(x) + 0 \\
&= \frac{1}{2}\sum_x p(x) \log \gamma(x)
\end{align*}

Similarly, 
\[\mathbb{E}_{XL} (1-L) \log(1-\gamma(X)) = \frac{1}{2}\sum_x q(x) \log (1-\gamma(x))
\] 
Using these in the expression for $\mathrm{BCE}(\gamma)$, we obtain 
\[
\mathrm{BCE}(\gamma) = - \frac{1}{2} \left( \sum\limits_{x \in \mathcal{X}} p(x) \log \gamma(x) + q(x) \log(1-\gamma(x)) \right)
\]
The point-wise minimizer $\gamma^*$ of $\mathrm{BCE}(\gamma)$ gives $\frac{\gamma^*(x)}{1-\gamma^*(x)} = \frac{p(x)}{q(x)}$.
\end{proof}

\begin{lemma}[Function Approximation]
Given $\epsilon' > 0$, $\ \exists \, \tilde{\theta} \in \Theta$ such that
\[
\mathrm{BCE}(\gamma_{\tilde{\theta}}) \leq  \mathrm{BCE}(\gamma^*) + \frac{\epsilon'}{2}
\]
\label{Lemma-Approximation}
\end{lemma}

\begin{proof}
The last layer of the neural network being sigmoid (followed by clipping to $[\tau, 1-\tau ]$) ensures that the outputs are bounded. So by the Universal Function Approximation Theorem for multi-layer feed-forward neural networks \citep{hornik}, $\exists$ parameter $\tilde{\theta}$ such that $| \gamma^* - \gamma_{\tilde{\theta}}| \leq \epsilon'' \, \forall \, x $, where $\gamma_{\tilde{\theta}}$ is the estimated classifier prediction function with parameter $\tilde{\theta}$. So,
\[
|\mathrm{BCE}(\gamma_{\tilde{\theta}}) - \mathrm{BCE}(\gamma^*)| \leq \frac{1}{\tau} \epsilon''
\]
since $\log$ is Lipshitz continuous with constant $\frac{1}{\tau}$. Choose $\epsilon'' = \frac{\epsilon'\tau}{2}$ to complete the proof.

\end{proof}

\begin{lemma}[Generalization] 

Given $\epsilon', \delta >0$, $\forall \, n \geq \frac{18 M^2}{\epsilon'^2}(h \log(96 KL \sqrt{d}/\epsilon') + \log(2/\delta))$, such that with probability at least $1-\delta$, we have 
\[
\mathrm{BCE}_n(\gamma_{\hat{\theta}}) \leq  \mathrm{BCE}(\gamma_{\tilde{\theta}}) + \frac{\epsilon'}{2}
\]
\label{Lemma-Generalization}
\end{lemma}

\begin{proof}
Let $\hat{\theta} \leftarrow \arg\min\limits_{\theta} \mathrm{BCE}^{\mathrm{ERM}}_n (\gamma_{\theta})$. 

From Hoeffding's inequality, 
\[
Pr \left(| \mathrm{BCE}^{\mathrm{ERM}}_n(\gamma_{\theta}) - \mathrm{BCE}(\gamma_{\theta})| \geq \mu \right) \leq 2\exp \left( \frac{-2n\mu^2}{M^2} \right)
\]
where $M = \log \left( \frac{1-\tau}{\tau} \right)$.

Similarly, for the test samples, 
\begin{equation}
Pr \left(| \mathrm{BCE}_n(\gamma_{\theta}) - \mathrm{BCE}(\gamma_{\theta})| \geq \mu \right) \leq 2\exp \left( \frac{-2n\mu^2}{M^2} \right)
\label{eq-hoeff-test}
\end{equation}

We want this to hold for all parameters $\theta \in \Theta$. This is obtained using the covering number of the compact domain $\Theta \subset \mathbb{R}^h$. We use small balls $B_r(\theta_j)$ of radius $r$ centered at $\theta_j$ so that $ \Theta \subset \cup_j B_r(\theta_j)$ The covering number $\kappa(\Theta, r)$ is finite as $\Theta$ is compact and is bounded as 
\[
\kappa(\Theta, r) \leq \left( \frac{2 K \sqrt{h}}{r} \right)^h
\]
Using the union bound on these finite hypotheses, 
\begin{multline}
Pr \left( \max\limits_{\theta} | \mathrm{BCE}^{\mathrm{ERM}}_n(\gamma_{\theta}) - \mathrm{BCE}(\gamma_{\theta})| \geq \mu \right) \\ \leq 2 \kappa(\Theta, r) \exp \left( \frac{-2n\mu^2}{M^2} \right)
\label{eq-union-bound}
\end{multline}

Choose $r = \frac{\mu}{8L}$ \citep{mohri2018foundations}. Solving for number of samples $n$ with $ 2 \kappa(\Theta, r) \exp \left( \frac{-2n\mu^2}{M^2} \right) \leq \delta$, we obtain $n \geq  \frac{M^2}{2\mu^2}(h \log(16 KL \sqrt{d}/\mu) + \log(2/\delta)  )$.

So for $n \geq \frac{M^2}{2\mu^2}(h \log(16 KL \sqrt{d}/\mu) + \log(2/\delta))$, with probability at least $1-\delta$, 
\begin{align*}
\mathrm{BCE}_n(\gamma_{\hat{\theta}}) & \overset{(a)}{\leq} \mathrm{BCE}(\gamma_{\hat{\theta}}) + \mu \overset{(b)}{\leq} \mathrm{BCE}^{\mathrm{ERM}}_n(\gamma_{\hat{\theta}}) + 2\mu \\
&\overset{(c)}{\leq} \mathrm{BCE}^{\mathrm{ERM}}_n(\gamma_{\tilde{\theta}}) + 2\mu \overset{(d)}{\leq} \mathrm{BCE}(\gamma_{\tilde{\theta}}) + 3\mu
\end{align*}

$ (a) $ follows from (\ref{eq-hoeff-test}). $(b)$ and $(d)$ follow from (\ref{eq-union-bound}). $(c)$ is due to the fact that $\hat{\theta}$ is the minimizer of train loss. Choosing $\mu = \epsilon'/6$ completes the proof.
\end{proof}

\begin{lemma}[Convergence to minimizer]

\label{Lemma-conv-min} Given $\epsilon' > 0$, $\exists \, \eta \left (=(1-\tau)\sqrt{\frac{2 \lambda(\mathcal{X}) \epsilon'}{\alpha}} \right) > 0$ such that whenever $ \mathrm{BCE}(\gamma_{\theta}) - \mathrm{BCE}(\gamma^*) \leq   \epsilon'$, we have 
\[
\| \vec{\gamma_{\theta}} - \vec{\gamma^*} \|_1 \leq \eta
\] 
where $\vec{\gamma} = [ \gamma(x) ]_{x \in \mathcal{X}} $ and $ \lambda(\mathcal{X})$ is the Lebesgue measure of compact set $\mathcal{X} \subset \mathbb{R}^{d_x}$.
\end{lemma}

\begin{proof}

\[
\mathrm{BCE}(\gamma) = - \frac{1}{2} \left( \sum\limits_{x \in \mathcal{X}} p(x) \log \gamma(x) + q(x) \log(1-\gamma(x)) \right)
\]
is $\alpha'$-strongly convex as a function of $\vec{\gamma}$ under Assumption (A2), where $\alpha' = \frac{\alpha}{(1-\tau)^2}$. So $ \forall \gamma \,,  \frac{\partial^2\mathrm{BCE}}{\partial \gamma(x_k) \partial \gamma(x_l)} \geq \alpha'$ for $k = l$ and $0$ otherwise. 
Using the Taylor expansion for strongly convex functions, we have
\begin{multline*}
\mathrm{BCE}(\vec{\gamma_{\theta}}) \geq \mathrm{BCE}(\vec{\gamma^*}) + \langle {\nabla\mathrm{BCE}(\vec{\gamma^*}), \vec{\gamma_{\theta}} - \vec{\gamma^*}} \rangle \\+ \frac{\alpha'}{2} \| \vec{\gamma_{\theta}} - \vec{\gamma^*} \|_2^2
\end{multline*}
Since $\vec{\gamma^*}$ is the minimizer, $ \nabla\mathrm{BCE}(\vec{\gamma^*}) = 0$. So,
\begin{align*}
&\| \vec{\gamma^*} - \vec{\gamma_{\theta}} \|_2 \\
&\leq (1-\tau) \sqrt{\frac{2}{\alpha}\left( \mathrm{BCE}(\vec{\gamma_{\theta}}) - \mathrm{BCE}(\vec{\gamma^*}) \right)} \\
&\implies \| \vec{\gamma^*} - \vec{\gamma_{\theta}} \|_2 \leq (1-\tau)\sqrt{\frac{2}{\alpha}\epsilon'}
\end{align*}
From Holder's inequality in finite measure space,
\begin{align*}
\| \vec{\gamma^*} - \vec{\gamma_{\theta}} \|_1 &\leq \sqrt{\lambda(\mathcal{X})} \| \vec{\gamma^*} - \vec{\gamma_{\theta}} \|_2  \\
&\leq (1-\tau)\sqrt{\frac{2}{\alpha}\lambda(\mathcal{X})\epsilon'} = \eta
\end{align*}

\end{proof}

\begin{lemma}[Estimation from Samples]
Given $\epsilon > 0$, for any classifier with parameter $\theta \in \Theta$, $\exists \, n \in \mathbb{N} $ such that with probability $1$,
\[
|I^{\gamma_{\theta}}_n (U;V) - I^{\gamma_{\theta}} (U;V)| \leq \frac{\epsilon}{2}
\] 
\label{Lemma-Estimation}
\end{lemma}
\begin{proof}
We denote the empirical estimates as $\mathop{\mathbb{E}}\limits_{x \sim p_n}(\cdot)$ and $\mathop{\mathbb{E}}\limits_{x \sim q_n}(\cdot)$ respectively. The proof essentially relies on the empirical mean of functions of independent random variables converging to the true mean. More specifically, we consider the functions $f^{\theta}(x) = \log \frac{\gamma^{\theta}(x)}{1-\gamma^{\theta}(x)}$ and $g^{\theta}(x) = \frac{\gamma^{\theta}(x)}{1-\gamma^{\theta}(x)}$. Since $\gamma(x) \in [\tau, 1-\tau]$, both $f(x)$ and $g(x)$ are bounded. ($f \in [\log \frac{\tau}{1-\tau}, \log \frac{1-\tau}{\tau}]$ and $g \in [\frac{\tau}{1-\tau},  \frac{1-\tau}{\tau}]$). Functions of independent random variables are independent. Also, since the functions are bounded, they have finite mean and variance. Invoking the law of large numbers, $\exists \, n \geq n'_1(\epsilon)$ such that with probability 1
\begin{equation}
| \mathop{\mathbb{E}}\limits_{x\sim p_n} f^\theta -  \mathop{\mathbb{E}}\limits_{x\sim p} f^\theta | \leq \frac{\epsilon}{4}
\label{eq-f-emp}
\end{equation}

and $\exists \, n \geq n'_2(\epsilon)$ such that with probability 1
\begin{equation}
| \mathop{\mathbb{E}}\limits_{x\sim q_n} g^\theta -  \mathop{\mathbb{E}}\limits_{x\sim q} g^\theta | \leq \frac{\epsilon \tau}{4 (1-\tau)}
\label{eq-g-emp}
\end{equation}

Then, for $n \geq \max(n'_1(\epsilon), n'_2(\epsilon))$, we have with probability 1
\begin{align*}
&|I^{\gamma_{\theta}}_n (U;V) - I^{\gamma_{\theta}} (U;V)| \\
&\leq | \mathop{\mathbb{E}}\limits_{x\sim p_n} f^\theta -  \mathop{\mathbb{E}}\limits_{x\sim p} f^\theta | + 
| \log \mathop{\mathbb{E}}\limits_{x\sim q_n} g^\theta -  \log \mathop{\mathbb{E}}\limits_{x\sim q} g^\theta | \\
&\leq |\mathop{\mathbb{E}}\limits_{x\sim p_n} f^\theta -  \mathop{\mathbb{E}}\limits_{x\sim p} f^\theta | + 
\frac{1-\tau}{\tau}|\mathop{\mathbb{E}}\limits_{x\sim q_n} g^\theta -  \mathop{\mathbb{E}}\limits_{x\sim q} g^\theta | \\
& = \frac{\epsilon}{4} + \frac{\epsilon}{4} = \frac{\epsilon}{2}
\end{align*}
where in the last inequality, we use the Lipschitz constant for $\log$ with the bounded function $g$ as argument. 
\end{proof}

\subsection*{Proof of Theorem \ref{re-theorem-class-mi-consistent}} \label{theo_ccmi}

Using Proposition \ref{prop-1}, $I^{\gamma^*}(U;V) = I(U;V)$, where $\gamma^*$ is the unique global minimizer of $\mathrm{BCE}(\gamma)$.

The empirical risk minimizer of $\mathrm{BCE}$ loss is $\hat{\theta}$. For a rich enough class $\Theta$ and large enough samples $n$, Lemma \ref{Lemma-Generalization} and Lemma \ref{Lemma-Approximation} combine to give $\mathrm{BCE}_n(\gamma_{\hat{\theta}}) - \mathrm{BCE}(\gamma^*) \leq \epsilon'$. Applying Lemma \ref{Lemma-conv-min} with $\epsilon' = \frac{\alpha}{8 \lambda(\mathcal{X})} \left( \frac{\eta}{\beta(1-\tau} \right)^2$, we have $\| \vec{\gamma^*} - \vec{\gamma}_{\hat{\theta}} \|_1 \leq \frac{\eta}{2 \beta}$. This further implies that
\begin{equation}
\mathop{\mathbb{E}}\limits_{x\sim p}| \gamma^* - \hat{\gamma}_{\hat{\theta}}| \leq \frac{\eta}{2}  
\label{ineq-p}
\end{equation}
and 
\begin{equation}
\mathop{\mathbb{E}}\limits_{x\sim q}| \gamma^* - \hat{\gamma}_{\hat{\theta}}| \leq \frac{\eta}{2} 
\label{ineq-q}
\end{equation}

We now compute the Lipschitz constant for $f = \log \frac{\gamma}{1-\gamma}$ as a function of $\gamma$, which links the classifier predictions to Donsker-Varadhan representation. 
\[
|f^* - \hat{f}^{\hat{\theta}}| = | \log \frac{\gamma^*}{1-\gamma^*} - \log \frac{\hat{\gamma}_{\hat{\theta}}}{1-\hat{\gamma}_{\hat{\theta}}}| 
\leq \frac{1}{\tau^2} |\gamma^* -  \hat{\gamma}_{\hat{\theta}}  |
\]
and 
\[
|e^{f^*} - e^{\hat{f}^{\hat{\theta}}}| = |  \frac{\gamma^*}{1-\gamma^*} - \frac{\hat{\gamma}_{\hat{\theta}}}{1-\hat{\gamma}_{\hat{\theta}}}| 
\leq \frac{1}{\tau^2} |\gamma^* -  \hat{\gamma}_{\hat{\theta}}  |
\]

For $\gamma \in  [\tau, 1-\tau ]$, the function $f \in [\log\frac{\tau}{1-\tau}, \log\frac{1-\tau}{\tau}] $ is continuous and bounded with Lipschitz constant $\frac{1}{\tau^2}$. So, using (\ref{ineq-p}) and (\ref{ineq-q}),

\[
\mathop{\mathbb{E}}\limits_{x\sim p}| f^* - \hat{f}^{\hat{\theta}}| \leq \frac{1}{\tau^2}\frac{\eta}{2} \,\,\, \mathrm{and} \,\, \mathop{\mathbb{E}}\limits_{x\sim q}| e^{f^*} - e^{\hat{f}^{\hat{\theta}}}| \leq \frac{1}{\tau^2}\frac{\eta}{2}  
\]
Finally, from the Donsker-Varadhan representation \ref{DV-bound}, 
\begin{align}
&|I(U;V) - I^{\gamma_{\hat{\theta}}} (U;V)| \leq |\mathop{\mathbb{E}}\limits_{x\sim p} f^* - \mathop{\mathbb{E}}\limits_{x\sim p}  \hat{f}^{\hat{\theta}} | + \notag \\
&|\log \mathop{\mathbb{E}}\limits_{x\sim q} e^{f^*} - \log \mathop{\mathbb{E}}\limits_{x\sim q}  e^{\hat{f}^{\hat{\theta}}} | \notag \\
&\leq \mathop{\mathbb{E}}\limits_{x\sim p}| f^* - \hat{f}^{\hat{\theta}}| + \mathop{\mathbb{E}}\limits_{x\sim q}| e^{f^*} - e^{\hat{f}^{\hat{\theta}}}| \notag \\
&= \frac{\eta}{2\tau^2} + \frac{\eta}{2\tau^2} = \frac{\eta}{\tau^2}  
\label{donsker-close} 
\end{align}

where we use the inequality $\log(t) \leq t-1$ coupled with the fact that $\mathop{\mathbb{E}}\limits_{x\sim q} e^{f^*} = 1$.
Given $\epsilon > 0$, we choose $\eta = \tau^2 \frac{\epsilon}{2} $. 

To complete the proof, we combine the above result (\ref{donsker-close}) with Lemma \ref{Lemma-Estimation} using Triangle Inequality,
\begin{align*}
&|I^{\gamma_{\hat{\theta}}}_n (U;V) - I(U;V) | \\
 \leq 
&|I^{\gamma_{\hat{\theta}}}_n (U;V) - I^{\gamma_{\hat{\theta}}} (U;V)| + |I^{\gamma_{\hat{\theta}}}(U;V) - I(U;V)| \\
&\frac{\epsilon}{2} + \frac{\epsilon}{2} =  \epsilon
\end{align*}

\begin{corollary}
CCMI is consistent. 
\label{cor-ccmi}
\end{corollary}

\begin{proof}
For each individual MI estimation, we can obtain the classifier parameter $\theta_1$(resp. $\theta_2$) $\in \Theta$ such that Theorem \ref{theorem-class-mi-consistent} holds with approximation accuracy $\epsilon/2$. So, $\exists n \geq n_1(\epsilon/2)$ such that with probability at least $1-\delta$
\[
|\hat{I}_n^{\gamma_{\theta_1}}(X;YZ) - I(X;YZ)| \leq \frac{\epsilon}{2}
\]
and $ n \geq n_2(\epsilon/2)$ such that with probability at least $1-\delta$
\[
|\hat{I}_n^{\gamma_{\theta_2}}(X;Z) - I(X;Z)| \leq \frac{\epsilon}{2}
\]

Using Triangle inequality, for $n \geq \max(n_1, n_2)$, with probability at least $1-\delta$, we have
\begin{align*}
&|\hat{I}_n(X;Y|Z) - I(X;Y|Z)| \\
& = |\hat{I}_n^{\gamma_{\theta_1}}(X;Y,Z) - \hat{I}_n^{\gamma_{\theta_2}}(X;Z) - I(X;Y,Z) + I(X;Z)| \\
&\leq |\hat{I}_n^{\gamma_{\theta_1}}(X;Y,Z) -I(X;Y,Z)| + |\hat{I}_n^{\gamma_{\theta_2}}(X;Z) -I(X;Z)| \\
&\leq \frac{\epsilon}{2} + \frac{\epsilon}{2} = \epsilon
\end{align*}

\end{proof}

\begin{theorem}[Theorem \ref{theorem-true-lower-bound} restated]
The finite sample estimate from Classifier-MI is a lower bound on the true MI value with high probability, i.e., given $n$ test samples and the trained classifier parameter $\hat{\theta}$, we have for $\epsilon > 0$
\[
Pr(I(U;V) + \epsilon \geq I^{\gamma_{\hat{\theta}}}_n (U;V) ) \geq 1 - 2\exp(-Cn)
\]
where $C$ is some constant independent of $n$ and the dimension of the data.
\label{re-theorem-true-lower-bound}
\end{theorem} 

\begin{proof}
\begin{multline*}
I(U;V) = \max\limits_{\gamma} I^{\gamma} (U;V)) \geq \max\limits_{\theta} I^{\gamma_{\theta}} (U;V)) \geq I^{\gamma_{\hat{\theta}}} (U;V)) 
\end{multline*}
We apply one-sided Hoeffding's inequality to (\ref{eq-f-emp}) and (\ref{eq-g-emp}) with given $\epsilon>0$, 
\begin{multline*}
Pr(\mathop{\mathbb{E}}\limits_{x\sim p_n} f^{\hat{\theta}} -  \mathop{\mathbb{E}}\limits_{x\sim p} f^{\hat{\theta}}  \leq \frac{\epsilon}{2}) \\
\geq  1 - \exp \left( -\frac{n\epsilon^2}{8 (\log ((1-\tau)/\tau ))^2}    \right) \\
= 1 - \exp(-C_1 n\epsilon^2)
\end{multline*}

\begin{multline*}
Pr \left( \mathop{\mathbb{E}}\limits_{x\sim p} g^{\hat{\theta}}  - \mathop{\mathbb{E}}\limits_{x\sim p_n} g^{\hat{\theta}} \leq \frac{\epsilon \tau}{2 (1-\tau)} \right) \\
\geq 1 - \exp \left( -\frac{n\epsilon^2}{2} \left( \frac{\tau}{1-\tau} \right)^4   \right) = 1 - \exp(-C_2 n\epsilon^2)
\end{multline*}

\[
Pr \left( I^{\gamma_{\hat{\theta}}}_n (U;V)) \leq   I^{\gamma_{\hat{\theta}}} (U;V)) + \epsilon  \right) \geq 1 - 2\exp(-Cn)
\]
where $C = \epsilon^2 \min(C_1, C_2) $.

\end{proof}

\end{document}